%% file: estimating.tex
\documentclass[11pt]{article}

\usepackage{psfrag,amsmath,amsbsy,amsfonts,amssymb,amsthm,fullpage}
\usepackage{graphicx}
\usepackage{algorithm,algorithmic,mathtools}
\usepackage{natbib}
\usepackage{color}
\usepackage{tikz}
\usetikzlibrary{calc,shapes}
\usepackage{subfig}
\usepackage{hyperref}
\usepackage{multirow}

\newtheorem{theorem}{Theorem}
\newtheorem{lemma}[theorem]{Lemma}

\newcommand{\dintervals}{{\mathcal{I}(d)}}
\newcommand{\dist}{{\mathrm{dist}}}
\newcommand{\interval}{{\mathcal{I}}}

\newcommand{\calC}{{\mathcal{C}}}
\newcommand{\calD}{{\mathcal{D}}}
\newcommand{\calS}{{\mathcal{S}}}
\newcommand{\property}{{\mathcal{P}}}
\newcommand{\D}{{\mathrm{d}}}
\newcommand{\knn}{{\text{$k$-NN}}}
\newcommand{\mv}{{\text{$k$-NN}^{\mathrm{hard}}}}
\newcommand{\thard}{\mathcal{T}^{\mathrm{hard}}}
\newcommand{\tsoft}{\mathcal{T}^{\mathrm{soft}}}
\newcommand{\knnsoft}{{\text{$k$-NN}^{\mathrm{soft}}}}
\newcommand{\loss}{{\mathrm{loss}}}
\newcommand{\calE}{{\mathcal E}}
\newcommand{\udi}{{\mathrm{Int}}}
\newcommand{\pt}{{\mathrm{PT}}}
\newcommand{\tot}{{\mathrm{TT}}}
\newcommand{\da}{{\mathrm{DA}}}
\newcommand{\qu}{_{\mathrm{query}}}
\newcommand{\ac}{_{\mathrm{active}}}

\newcommand{\ap}{{\mathrm{Comp}}}
\newcommand{\aga}{{\mathrm{AGA}}}
\newcommand{\cga}{{\mathrm{CGA}}}
\newcommand{\calU}{{\mathcal{U}}}

\newcommand{\calA}{{\mathcal{A}}}
\newcommand{\error}{{\mathrm{error}}}
\newcommand{\calT}{{\mathcal{T}}}
\newcommand{\err}{{\mathrm{err}_1}}
\newcommand{\calB}{{\mathcal B}}

\title{%Estimating the Accuracy of Learning Algorithms\\ {\small [or]}\\
Active Tolerant Testing%\\ {\small [or]}\\
}

\author{
   Avrim Blum\\
   \scriptsize Toyota Technological Institute at Chicago\\
   \scriptsize Chicago, USA\\
   \scriptsize \emph{avrim@ttic.edu}
   \and Lunjia Hu\\
   \scriptsize Institute for Interdisciplinary Information Sciences\\
   \scriptsize Tsinghua University\\
   \scriptsize Beijing, China\\
   \scriptsize \emph{hulj14@mails.tsinghua.edu.cn}
}

\begin{document}
\maketitle

\begin{abstract}
In this work, we give the first algorithms for tolerant testing of nontrivial classes in the active model: estimating the distance of a target function to a hypothesis class $\calC$ with respect to some arbitrary distribution $\calD$, using only a small number of label queries to a polynomial-sized pool of unlabeled examples drawn from $\calD$.   Specifically, we show that for the class $\calC$ of unions of $d$ intervals on the line, we can estimate the error rate of the best hypothesis in the class to an additive error $\epsilon$ from only $O(\frac{1}{\epsilon^6}\log \frac{1}{\epsilon})$ label queries to an unlabeled pool of size $O(\frac{d}{\epsilon^2}\log \frac{1}{\epsilon})$.  The key point here is the number of labels needed is independent of the VC-dimension of the class. This extends the work of  \citet{BBBY12} who solved the {\em non}-tolerant testing problem for this class (distinguishing the zero-error case from the case that the best hypothesis in the class has error greater than $\epsilon$).  

We also consider the related problem of estimating the performance of a given learning algorithm $\calA$ in this setting.  That is, given a large pool of unlabeled examples drawn from distribution $\calD$, can we, from only a few label queries, estimate how well $\calA$ would perform if the entire dataset were labeled?   We focus on $k$-Nearest Neighbor style algorithms, and also show how our results can be applied to the problem of hyperparameter tuning (selecting the best value of $k$ for the given learning problem).
\end{abstract}
\input{introduction}
\input{results}

\input{preliminaries}

\input{additive}
\input{interval}
\input{neighbor}
\input{acknowledgement}

%The sufficient set of conditions include incoherent  dictionary elements  and the number of samples needs to scale only linearly in the number of dictionary elements.

%\paragraph{Keywords: }Recursive teaching dimension; VC dimension; Recursive teaching model.
%\clearpage
%\newpage
\bibliographystyle{plainnat}
\bibliography{references}%\balancecolumns
\appendix
\input{union}
\end{document}

%% file: introduction.tex
\section{Introduction}
Suppose you are about to embark on a project to label a large quantity of data, such as medical images or street scenes.  Your intent is to then feed this data into your favorite learning algorithm for, say, a medical diagnosis or robotic car application.  Before embarking on this project, can you, from just a few labeled examples, estimate {\em how well} your algorithm can be expected to perform when trained on the large sample?  
We consider this question in two related  contexts.  

\paragraph{Tolerant testing:} The first context we consider is that of tolerant testing, or approximating the distance of a target function $f$ to a hypothesis class $\calC$.  Specifically, consider a hypothesis class $\calC$ of VC-dimension $d$, where $d$ should be thought of as large.  If we wish to find the $\epsilon$-approximately-best hypothesis in $\calC$, we will need roughly $O(d/\epsilon)$ labeled examples in the (realizable) case that $f \in \calC$, or $O(d/\epsilon^2)$ labeled examples in the (agnostic) case that $f \not\in \calC$.   However, if we just want to estimate the error rate of the best hypothesis in $\calC$ (rather than to {\em find} the hypothesis), can we do this from less data?

%There are two scenarios in which we want to know the accuracy of a learning algorithm beforehand. The first one is when the algorithm does Empirical Risk Minimization over a hypothesis class, where the bottleneck for accuracy is the best function in the underlying hypothesis class. The second one is when the bottleneck for accuracy is the limited amount of data.

The problem of determining {\em whether} one will be able to learn well using a given hypothesis class $\calC$ using substantially  less labeled data than needed to actually (attempt to) learn well using that class  is the problem of {\em passive} and {\em active} property testing, studied by \citet{KR98} and \citet{BBBY12}.  That work considers the problem of distinguishing the case that (a)  the target function $f$ belongs to class $\calC$ from (b) the target function $f$ is $\epsilon$-far from any concept in $\calC$ with respect to the underlying data distribution $\calD$.   For instance, suppose our data consists of points $x$ on the real line, labeled by $f$ as positive or negative, and we are interested in learning using the class $\calC$ consisting of unions of $d$ intervals.  This class has VC-dimension $2d$ and so would require $\Omega(d)$ labeled examples to learn.  However, \citet{BBBY12} show that in the active testing framework (one can sample $poly(d)$ {\em unlabeled} examples for free and then query for the labels of a small number of those examples), one can solve the testing problem using only a constant number of label queries (when $\epsilon$ is constant), independent of $d$.

One limitation of these results, however, is that they do not guarantee to give a meaningful answer when the target function is ``almost'' in the  class $\calC$.  For instance, suppose $f$ can be perfectly described by a union of 10,000 intervals but is $\epsilon/2$-close to a union of 100 intervals.  Then we would like a tester that can say ``good enough'' at $d=100$ rather than telling us that we need $d=10,000$.  The tester of \citet{BBBY12}, unfortunately, seems to require $f$ to be $O(\epsilon^3)$-close to a union of $d$ intervals in order to guarantee an output of YES, which is much less than $\epsilon$.

In this work, we give algorithms for such {\em tolerant testing} \citep{PRR06} for the case of unions of intervals and a few related classes.  We can distinguish the case that the best function in $\calC$ has error rate $\geq 2\epsilon$ from the case that the best function in $\calC$ has error rate $\leq \epsilon$, and more generally we can estimate the error rate $\alpha$ of the best function in the class up to $\pm \epsilon$.  Thus, for the first time, from a small number of label queries, we can solve the property-testing analog of the notion of agnostic learning.

One point we wish to make up front: while the classes of functions we consider are fairly simple, such as unions of intervals on the line, we are operating in a challenging model.  We would like algorithms that work for any (unknown) underlying data distribution $\calD$, not just the uniform distribution, {\em and} we want algorithms that only query for labels  from among examples seen in a poly-sized sample of unlabeled data drawn from $\calD$ rather than querying arbitrary points in the input space.  These are important conditions for being able to use property testing for machine learning problems.  

\paragraph{Algorithm estimation:}
The second context we consider is that we are given a learning algorithm $\calA$ and a large unlabeled sample $S$ of $N$ examples drawn from distribution $\calD$.  If we were to label all $N$ examples of $S$ and feed them into algorithm $\calA$, then $\calA$ would produce some hypothesis (call it $h_S$) with some error rate $\alpha$.   What we would like to do is, by labeling only very few examples in $S$, and perhaps a few additional examples drawn from $\calD$, to estimate the value of $\alpha$ (so that we can determine whether our project of labeling all examples in $S$ is worthwhile).   

To get a feel for this problem, one algorithm $\calA$ for which this task is easy is 1-Nearest Neighbor (1-NN).  This algorithm would produce a hypothesis $h_S$ that on any given query point $x$ predicts the label of the example $x' \in S$ that is nearest to $x$.  For this algorithm, we can easily estimate the error rate of $h_S$ from just a few label queries by repeatedly drawing a random $x$ from $\calD$, finding the point $x' \in S$ that is closest to $x$, and then requesting the labels of $x$ and $x'$ to see if they agree.   We only need to repeat this process $O(1/\epsilon^2)$ times in order to estimate the error rate of $h_S$ to $\pm \epsilon$.  This works because $h_S$ is constructed, and makes predictions, in a very local way.\footnote{In contrast, note that estimating the error rate of this algorithm would require a large labeled sample if we only \emph{passively} receive labeled examples.  Specifically, suppose the distribution $\calD$ is uniform over $c$ clusters and the 1-NN algorithm aims to use $N=c\log\frac{c}{\delta}$ examples, so that with probability at least $1-\delta$, every cluster has at least one training example in it.  We want to distinguish the following two cases: either every cluster is pure but random so the error rate is roughly 0, or every cluster is 50/50 so the error rate is roughly $\frac{1}{2}$. To distinguish these cases, the tester needs to see at least two labels in the same cluster, implying an $\Omega(\sqrt c)=\Omega(\sqrt {N/\log N})$ sample complexity lower bound.}  In this work, we extend this to different forms of $k$-Nearest Neighbor algorithms, where the prediction on some point $x$ depends on the $k$ nearest examples in $S$, developing testers for which the number of queries {\em does not depend on $k$}.    This then allows us to use this for {\em hyperparameter tuning}: determining the (approximately) best value of $k \in \{1, \ldots, N\}$ for the given application.

%[[regarding the 3 slightly different settings:]]

We note that there are three natural but somewhat different ways to model the task of estimating the error rate of algorithm $\calA$ trained on dataset $S$. Let $\error(h_S)$ denote the error rate of hypothesis $h_S$ with respect to distribution $\calD$, and let $\hat\alpha$ be the output of the tester $\calT$ that estimates $\error(h_S)$. In the first model, we require that $\hat\alpha$ be a good estimate of $\error(h_S)$ with probability at least $\frac23$ for {\em any} training set $S$, even sets $S$ not drawn from $\calD$.  In the second model, we only require that $\calT$ be accurate when $S$ is drawn from $\calD$ (that is, the $\frac23$ probability is over both the internal randomness in $\calT$ and in the draw of $S$).  Finally, in the third model, $S$ is drawn from $\calD$ but $\calT$ does not have access to it: instead, $\calT$ has the ability to draw (a polynomial number of) fresh unlabeled examples and to query points from them.  That is,
%could take $S$ as its input (as in the first two models) or not (as in the third model). There are two types of randomness: the randomness of $S$ (iid from $\calD$) and the internal randomness of the tester $\calT$. 
\begin{enumerate}
\item In the first model, we require that $\forall S,\Pr_{\calT(S)}[|\hat\alpha-\error(h_S)|\leq\epsilon]\geq \frac{2}{3}$. 
\item In the second model, we require that $\Pr_{S,\calT(S)}[|\hat\alpha-\error(h_S)|\leq \epsilon]\geq \frac{2}{3}$. 
\item In the third model, we require that $\Pr_{\calT}[|\hat\alpha-\mathbb E_S[\error(h_S)]|\leq\epsilon]\geq \frac{2}{3}$. 
\end{enumerate}
Roughly, the first model is the hardest while the third model is the easiest.  
All our upper bounds and lower bounds in this paper apply to all three models with slight modifications, though for simplicity of presentation we focus on the second model throughout the paper.%first model for upper bounds (Sections \ref{subsec:knnsoft}, \ref{subsec:bestk}, \ref{subsec:mv}) and the second model for lower bounds (Section \ref{subsec:lower}) because the first model is more restraining than the second model.

%Though in this paper, we present our upper bound and lower bound results in the second model, they can be naturally generalized to all models with slight modifications.
\label{sec:intro}

%% file: results.tex
\section{Our Results}
\label{sec:results}

In this paper, we show (Theorem \ref{thm:main}) that in the active testing model \citep{BBBY12}, there is an algorithm that approximates the distance from a function to the class of unions of $d$ intervals on the line up to an additive error $\epsilon$ using $O(\frac{1}{\epsilon^6}\log\frac{1}{\epsilon})$ label queries on $O(\frac{d}{\epsilon^2}\log\frac{1}{\epsilon})$ unlabeled examples, even when the data distribution is unknown to the algorithm.

To achieve this result, we propose the notion of \emph{compositions of additive properties} (Section \ref{subsec:composition}) and prove the Composition Lemma (Lemma \ref{thm:additive}) that to approximate the distance to any composition of $m$ additive properties on a semi-uniform distribution up to an additive error $\epsilon$, we only need a distance approximation oracle for compositions of only $O(\frac{1}{\epsilon\mu^2}+\frac{1}{\epsilon^2})$ additive properties, though this may produce a bi-criteria approximation that depends on $\mu$. See Section \ref{subsec:composition} for definitions.

The Composition Lemma implies an $(\epsilon,\mu)$-bi-criteria distance approximation algorithm for unions of $d$ intervals on the uniform distribution over $[0,1]$ using $O((\frac{1}{\epsilon^3\mu^3}+\frac{1}{\epsilon^4\mu})\log\frac{1}{\epsilon})$ queries on $O(\frac{d}{\epsilon^2}\log\frac{1}{\epsilon})$ unlabeled examples; in particular, we estimate the error to an additive $\pm \epsilon$ and the number of intervals to a multiplicative factor $1+\mu$. We then show (Lemma \ref{lm:reductiontobicriteria}) how to remove the approximation in number of intervals and get a uni-criterion distance approximation algorithm.

To generalize the result to arbitrary unknown distributions, we show a general relationship between query testing and active testing for arbitrary distributions in Theorems \ref{thm:unlabeled} and \ref{thm:reductiontoquery}, which also improves a previous result in \citep{BBBY12} by showing that the unlabeled sample complexity of non-tolerant property testing for unions of $d$ intervals on arbitrary unknown distributions can be reduced to $O(\frac{d}{\epsilon}\log\frac{1}{\epsilon})$, from $O(\frac{d^2}{\epsilon^6})$ in their original paper.

For the $k$-Nearest Neighbor (\knn) algorithm with soft predictions and $p$th-power loss (the prediction on a point $x$ is the average label of the $k$ nearest examples to $x$ in an unlabeled sample of size $N$, and we use the $p$th-power loss to penalize mistakes) we show in Theorem \ref{thm:pth} that this loss can be estimated up to an additive error $\epsilon$ using $O(\frac{p}{\epsilon^2})$ queries on $N+O(\frac{1}{\epsilon^2})$ unlabeled examples, even when the data distribution is unknown to the tester. The same result also holds for Weighted Nearest Neighbor algorithms, where the prediction on a point $x$ is a weighted average of the labels of all the examples depending on their distances to $x$. For the $O(\frac{p}{\epsilon^2})$ query complexity upper bound, we show a matching lower bound (Theorem \ref{thm:pthpowerlower}). In the case where $k$ is a quantity to be optimized, we show an algorithm that finds an approximately-best choice of $k$ up to an additive error $\epsilon$ using roughly 
%$O(\frac{p^2\log N}{\epsilon^3}(\log\log N+\log p+\log\frac{1}{\epsilon}))$
$O(\frac{p^2\log N}{\epsilon^3})$ queries on roughly $N+O(\frac{p\log N}{\epsilon^3})$ unlabeled examples.  For $\knn$ with hard predictions (the prediction is a strict majority vote over the $k$ nearest neighbors), there is a simple algorithm for approximating its accuracy up to an additive error $\epsilon$ using $O(\frac{k}{\epsilon^2})$ queries on $N+O(\frac{1}{\epsilon^2})$ unlabeled examples. By a reduction (Theorem \ref{thm:reductionfrombandit}) from \emph{approximating the number of good arms} ($\aga$, see Section \ref{subsec:arm} for definition) in the stochastic multi-armed bandit setting, we show that the query complexity cannot be improved beyond $O(\frac{k}{\epsilon\log\frac{1}{\epsilon}})$ (Theorem \ref{thm:mv}), and cannot even be improved beyond the current $O(\frac{k}{\epsilon^2})$ upper bound if we assume the natural algorithm for $\aga$ is optimal in query complexity.

%In Section \ref{sec:preli}, we introduce important definitions and related previous results. In Section \ref{sec:additive}, we prove the Composition Lemma (Lemma \ref{thm:additive}). In Section \ref{sec:interval}, we prove Theorem \ref{thm:main} of active tolerant testing for unions of intervals. In Section \ref{sec:neighbor}, we consider problems of estimating the performance of $\knn$. In Appendix \ref{sec:union}, we extend the result of testing disjoint unions of properties from non-tolerant property testing \citep{BBBY12} to tolerant testing.

%% file: preliminaries.tex
\section{Preliminaries and Related Work}
\label{sec:preli}
\subsection{Property Testing, Tolerant Testing and Distance Approximation}
Suppose we have a ground set $X$ and a distribution $\calD$ over $X$. For any two binary functions $f,g\in\{0,1\}^X$, we define their distance to be $\dist_{\calD}(f,g)=\Pr_{x\sim \calD}[f(x)\neq g(x)]$. 

Suppose we also have a concept class $\calC\subseteq\{0,1\}^X$. Given a function $f\in\{0,1\}^X$ and a margin $\epsilon$ as input, the task of property testing $\pt_{\calD}(f,\epsilon)$ \citep{RS96} is to distinguish the case that $f$ belongs to class $\calC$ from the case that $f$ is $\epsilon$-far from $\calC$. In other words, $\forall f$,
\begin{enumerate}
\item if $f\in\calC$, the algorithm outputs ``YES'' with probability at least $\frac{2}{3}$;
\item if $\forall g\in\calC,\dist_{\calD}(f,g)>\epsilon$, the algorithm outputs ``NO'' with probability at least $\frac{2}{3}$.
\end{enumerate}
%A property testing algorithm may be randomized, and in this case, the goal is to output the correct answer with probability at least $\frac{2}{3}$. The success probability can be boosted to $1-\delta$ by repeating the algorithm for $O(\log\frac{1}{\delta})$ times and taking the majority.

The function $f$ can be given to the algorithm in many different ways. In the \emph{query testing} framework \citep{RS96}, the algorithm can query the value of $f(x)$ for any $x\in X$. In this framework, we say the algorithm has \emph{query access} to $f$, or has access to $f\qu$. The query complexity of the algorithm, as a function of $\frac{1}{\epsilon}$, is measured by the maximum number of queries needed by the algorithm.

\citet{BBBY12} argued that the query testing framework is not realistic for machine learning practice. They proposed the \emph{active testing} framework, in which the algorithm first requests $N$ unlabeled examples $x_1,x_2,\cdots,x_{N}\in X$ sampled independently according to $\calD$ and can only choose to query $f(x_i)$ for $1\leq i\leq N$. In this framework, we say the algorithm has \emph{active access} to $f$, or has access to $f\ac$. The maximum value of $N$, as a function of $\frac{1}{\epsilon}$, is called the \emph{unlabeled sample complexity}. The query complexity is still defined as a function of $\frac{1}{\epsilon}$ measuring the maximum number of queries needed by the algorithm.

\citet{GGR98} and \citet{KR98} studied an even more strict way of accessing $f$, called \emph{passive access}, in which the algorithm is given the label of an example chosen independently at random from $\calD$ for each query the algorithm makes. 

Tolerant testing $\tot_{\calD}(f,\alpha,\epsilon)$ \citep{PRR06} is a similar task to property testing. The only difference is that besides the margin $\epsilon$, we are given another parameter $\alpha$ as input, and we are asked to distinguish the case that $f$ is $\alpha$-close to $\calC$ from the case that $f$ is $(\alpha+\epsilon)$-far from $\calC$. %following two cases:
%\begin{enumerate}
%\item $\exists g\in\calC,\dist_{\calD}(f,g)\leq \alpha$;
%\item $\forall g\in\calC,\dist_{\calD}(f,g)>\alpha+\epsilon$.
%\end{enumerate} 
The query complexity of tolerant testing is still measured as a function of $\frac{1}{\epsilon}$ \citep{PRR06}. 

A natural generalization of tolerant testing is distance approximation, in which we are only given the function $f$ and the margin $\epsilon$ as input and required to output $\hat\alpha$ as an approximation of the distance from $f$ to $\calC$ up to an additive error $\epsilon$. More specifically, the goal of $\da_{\calD}(f,\epsilon)$ is to output $\hat\alpha$ such that $\forall\alpha$,
\begin{enumerate}
\item $\forall f$ such that $\exists g\in\calC,\dist_{\calD}(f,g)\leq \alpha$, it holds with probability at least $\frac{2}{3}$ that $\hat\alpha\leq \alpha+\epsilon$;
\item $\forall f$ such that $\forall g\in\calC,\dist_{\calD}(f,g)>\alpha$, it holds with probability at least $\frac{2}{3}$ that $\hat\alpha> \alpha-\epsilon$.
\end{enumerate}

%The success probability of a distance approximation algorithm can be boosted to $1-\delta$ by repeating it $O(\log\frac{1}{\delta})$ times and taking the median. 

Because for any $\calD$ and $\epsilon$, it's clear that a $\da_{\calD}(f,\frac{\epsilon}{2})$ algorithm implies a $\tot_{\calD}(f,\alpha,\epsilon)$ algorithm with the same query and unlabeled sample complexity \citep{PRR06}, we focus on distance approximation rather than tolerant testing throughout the paper.

%For all of the above examples, the distribution $\calD$ appears in the subscript meaning that the distribution is fixed and known to the algorithm. In this paper, we will consider a more general case, for example $\pt(f\qu,\epsilon,\calD)$, in which the distribution $\calD$ is arbitrarily given to the algorithm as input. When the algorithm has active or passive access to $f$, it's possible to consider an even more general case, for example $\pt(f\ac,\epsilon)$, in which the distribution $\calD$ is unknown to the algorithm (but implicitly given to the algorithm when accessing $f$).

Obviously, as pointed out by \citet{BBBY12},  a $\pt_{\calD}(f\ac,\epsilon)$ algorithm implies a $\pt_{\calD}(f\qu,\epsilon)$ algorithm with the same query complexity when $\calD$ is known to the algorithm, since it can always then create unlabeled data on its own; this also holds for $\tot$ and $\da$.  Here, we show a theorem in the reverse direction for bounds that are worst-case over distributions $\calD$.  
%that has the reverse direction of the previous statement when $\calD$ is arbitrary rather than fixed. 
Specifically, we show in Theorem \ref{thm:unlabeled} that a $\pt_{\calD}(f\qu,\frac{\epsilon}{2})$ algorithm can induce a $\pt_{\calD}(f\ac,\epsilon)$ algorithm with (except for a constant factor) the same query complexity and reasonable unlabeled sample complexity, under the assumption that the $\pt_{\calD}(f\qu,\frac{\epsilon}{2})$ algorithm never queries examples outside the support of $\calD$, which holds in all normal cases. We extend the theorem to $\da$ in Theorem \ref{thm:reductiontoquery}.

\subsection{Unions of $d$ Intervals}
We use $\interval(d)\subseteq\{0,1\}^{\mathbb{R}}$ to denote the class of functions $f\in\{0,1\}^{\mathbb{R}}$ satisfying that $f^{-1}(1)$ can be written as a union of at most $d$ intervals. Note that for $d\in\mathbb{N}$, the VC-dimension of $\interval(d)$ is $2d$.

We use $\interval_{\calD}(d,\alpha)$ to denote the class of functions that are $\alpha$-close to $\interval(d)$, i.e. $\interval_{\calD}(d,\alpha)=\{f\in\{0,1\}^{\mathbb{R}}:\exists g\in\interval(d),\dist_{\calD}(f,g)\leq\alpha\}$. Using this notation, property testing for unions of $d$ intervals is to distinguish $f\in\interval(d)$ and $f\notin\interval_{\calD}(d,\epsilon)$. %In the rest of the paper, when $\mathcal{D}$ is the uniform distribution on unit interval $[0,1]$, we omit it for short.

In previous work, \citet{KR98} showed that in the query testing framework, when $\calD$ is the uniform distribution over $[0,1]$, there is a bi-criteria testing algorithm that can distinguish $f\in\dintervals$ and $f\notin\mathcal{I}_{\calD}(\frac{d}{\epsilon},\epsilon)$ using $O(\frac{1}{\epsilon})$ queries. %Their algorithm can be easily generalized to a bi-criteria tolerant testing algorithm that distinguishes $f\in\mathcal{I}(d,\alpha_1)$ and $f\notin\mathcal{I}(\frac{d}{O(\alpha_2-\alpha_2^2-\alpha_1)},\alpha_2)$ using $O(\frac{1}{(\alpha_2-\alpha_2^2-\alpha_1)^2})$ queries. Although their algorithm lies in the query testing framework, the queries they use are randomly chosen and can be generated in the active testing framework using $O(\frac{\sqrt{d}}{\epsilon^{1.5}})$ (or $O(\frac{\sqrt{d}}{(\alpha_2-\alpha_2^2-\alpha_1)^{2.5}})$ with tolerance) unlabeled samples. 
\citet{BBBY12} improved this work by showing that in the active testing framework, there is a uni-criterion testing algorithm that can distinguish $f\in\interval(d)$ and $f\notin\mathcal{I}_\calD(d,\epsilon)$ using $O(\frac{1}{\epsilon^4})$ queries on $O(\frac{\sqrt{d}}{\epsilon^5})$ unlabeled examples. Their algorithm can be generalized to a testing algorithm that distinguishes $f\in\mathcal{I}_\calD(d,\epsilon_1)$ and $f\notin\mathcal{I}_\calD(d,\epsilon)$ when $\epsilon_1=O(\epsilon^3)$ using the same number of queries and unlabeled examples. 
%When the distribution is uniform on $[0,1]$, their algorithm can be viewed as an algorithm in the \emph{query testing} framework implemented in the \emph{active testing} framework by continuous sampling until obtaining the data that the algorithm wants to query. 
They generalized the result from uniform distribution on $[0,1]$ to any unknown distribution by taking the advantage of unlabeled examples to approximate the CDF of the distribution to enough accuracy using $O(\frac{d^2}{\epsilon^6})$ unlabeled examples. This unlabeled sample complexity is improved to $O(\frac{d}{\epsilon}\log\frac{1}{\epsilon})$ in our paper (Section \ref{subsec:relationship}).

%Previous works \citep{KR98,BBBY12} on testing union of $d$ intervals reveal close relationship between the query testing framework and the active testing framework, which we will discuss in Section \ref{sec:relation}. 

%there is a tolerant testing algorithm that can distinguish $f\in\interval_{\calD}(d,\alpha)$ and $f\notin\interval_{\calD}(d,\alpha+\epsilon)$ using $O(\frac{1}{\epsilon^6}\log\frac{1}{\epsilon})$ queries on $O(\frac{d}{\epsilon^2}\log\frac{1}{\epsilon})$ unlabeled samples, even when the distribution $\calD$ is unknown to the algorithm.

\subsection{Composition of Additive Properties}
\label{subsec:composition}
\citet{BBBY12} showed that disjoint unions of testable properties are testable in the non-tolerant, active model. We extend their result to distance approximation in Appendix \ref{sec:union}. Here, we propose a more general notion of a certain concept class formed by composing smaller concept classes on disjoint ground sets.

Suppose we have $m$ disjoint ground sets $X_1,X_2,\cdots,X_m$ and on each $X_i$, we have a sequence of concept classes $\calC_i^0, \calC_i^1, \calC_i^2, \cdots\subseteq\{0,1\}^X$. Suppose $\calC_i^0\neq\emptyset$ for all $i$. We use $X$ to denote the disjoint union $\bigcup\limits_{i=1}^m X_i$. For any $d\geq 0$, we define a concept class $\property(d)$ on $X$ to be the class of functions $f\in\{0,1\}^X$ satisfying that $\exists k_1,k_2,\cdots,k_m\in\mathbb{N}$ s.t.
\begin{enumerate}
\item $\sum\limits_{i=1}^mk_i\leq d$;
\item $\forall 1\leq i\leq m, f|_{X_i}\in\calC_i^{k_i}$.
\end{enumerate}

We call $\property$ a \emph{composition of $m$ additive properties}. Note that $\property(0)=\{f\in\{0,1\}^X:\forall 1\leq i\leq m,f|_{X_i}\in\calC_i^0\}$, matching the definition of a disjoint union of properties in \citep{BBBY12}. Also note that $\property(0)\neq \emptyset$ because of the assumption that $\calC_i^0\neq \emptyset$ for all $i$. 

For a given $t\geq 0$, we define a composition $\property^t$ in the same way as $\property$ except that we further require every $k_i$ to be at most $t$, or, $\property^t$ is a composition of $m$ additive properties \emph{truncated by $t$}. %If $\property(0)=\emptyset$, i.e., for some $i$, it holds that $\calC_i^0=\emptyset$, we say the composition is \emph{degenerated}. In this case, if we define the new $\calC_i^{k}$ to be the old $\calC_i^{k+1}$ for $k=0,1,\cdots$, then the new $\property(d)$ becomes the old $\property(d+1)$ for all $d$. By repeating this process, as long as there exists some $d$ such that $\property (d)$ is non-empty, we can finally make $\property(0)$ to be non-empty, i.e., we can make the composition to be \emph{non-degenerated}.

For any distribution $\calD$ over $X$, we use $\property_{\calD}(d,\alpha)$ to denote functions that are $\alpha$-close to $\property(d)$ with respect to $\calD$, i.e. $\property_{\calD}(d,\alpha)=\{f\in\{0,1\}^X:\exists g\in\property(d),\dist_{\calD}(f,g)\leq\alpha\}$. Similarly, we define $\property^t_{\calD}(d,\alpha)=\{f\in\{0,1\}^X:\exists g\in\property^t(d),\dist_{\calD}(f,g)\leq\alpha\}$. We say $\calD$ is \emph{semi-uniform} if $\forall 1\leq i\leq m,\Pr_{x\sim \calD}[x\in X_i]=\frac{1}{m}$.

%\subsection{Composition with Truncation}
%\label{subsec:truncation}
%Suppose $m,X_i,\calC_i^k,X$ are defined as in Section \ref{subsec:composition}. 
%Suppose some $t\geq 0$ is chosen as the \emph{threshold of truncation}. For any $d\geq 0$, we define a concept class $\property^t(d)$ on $X$ to be the class of functions $f\in\{0,1\}^X$ satisfying that $\exists k_1,k_2,\cdots,k_m\in\mathbb{N}$ s.t.
%\begin{enumerate}
%\item $\forall 1\leq i\leq m, k_i\leq t$;
%\item $\sum\limits_{i=1}^mk_i\leq d$;
%\item $\forall 1\leq i\leq m, f|_{X_i}\in\calC_i^{k_i}$.
%\end{enumerate}
%We call $\property^t$ a composition of $m$ additive properties \emph{truncated by $t$}. For any distribution $\calD$ over $X$, we use $\property_{\calD}^t(d,\alpha)$ to denote $\{f\in\{0,1\}^X:\exists g\in\property^t(d)\text{ s.t. }\dist_{\calD}(f,g)\leq\alpha\}$.
\begin{lemma}
\label{lm:truncation}
Suppose the distribution $\calD$ is semi-uniform. We have $\property_{\calD}^t(d,\alpha)\subseteq\property_{\calD}(d,\alpha)\subseteq\property_{\calD}^t(d,\alpha+\frac{d}{tm})$.
\end{lemma}
\begin{proof}
$\property_{\calD}^t(d,\alpha)\subseteq\property_{\calD}(d,\alpha)$ is obvious. To see $\property_{\calD}(d,\alpha)\subseteq\property_{\calD}^t(d,\alpha+\frac{d}{tm})$, we note that for any $g\in\property(d)$, for each $i$ such that $k_i>t$, substituting a function in $\calC_i^0$ for $g|_{X_i}$ causes at most a $\frac{1}{m}$ increase in the distance from $f\in\property_\calD(d,\alpha)$ to $g$. An easy observation that $|\{i:k_i>t\}|\leq\frac{d}{t}$ given $\sum\limits_{i=1}^mk_i\leq d$ completes the proof. 
\end{proof}
An $(\epsilon,\mu)$-bi-criteria distance approximation algorithm $\ap_{\calD}(f,(\epsilon,\mu),d)$ for composition $\property$ of additive properties, is an algorithm that takes $f,\epsilon,\mu$ and $d$ as input and outputs $\hat\alpha$ such that $\forall \alpha$
\begin{enumerate}
\item $\forall f\in\property_{\calD}(d,\alpha)$, it holds with probability at least $\frac{2}{3}$ that $\hat\alpha\leq\alpha+\epsilon$;
\item $\forall f\notin\property_{\calD}((1+\mu)d,\alpha)$, it holds with probability at least $\frac{2}{3}$ that $\hat\alpha>\alpha-\epsilon$.
\end{enumerate}
%\subsection{Distance Approximation Oracle}
%Suppose $m,X_i,\calC_i^k,X,t,\widehat\property(d)$ are defined as in Section \ref{subsec:composition} and Section \ref{subsec:truncation}. Suppose we have a sequence of indices $1\leq i_1<i_2<\cdots i_l\leq m$. We use $\mathbf{i}$ to denote $(i_1,i_2,\cdots,i_l)$ for short. For any $d\geq 0$, we define a concept class $\widehat\property^{\mathbf i}(d)$ on $X$ to be the class of functions $f\in\{0,1\}^X$ satisfying that $\exists k_1,k_2,\cdots,k_l\in\mathbb{N}$ s.t.
%\begin{enumerate}
%\item $\forall 1\leq j\leq l,k_j\leq t$;
%\item $\sum\limits_{j=1}^lk_j\leq d$;
%\item $\forall 1\leq j\leq l,f|_{X_{i_j}}\in\calC_{i_j}^{k_j}$.
%\end{enumerate}
%Here, $\widehat\property^{\mathbf i}$ is called the \emph{sub-composition of $l$ additive properties truncated by $t$}. When $l=m$, we have $\widehat\property^{\mathbf i}(d)=\widehat\property(d)$.

%For any distribution $\calD$ over $X$, we use use $\widehat\property_{\calD}^{\mathbf i}(d,\alpha)$ to denote $\{f\in\{0,1\}^X:\exists g\in\widehat\property^{\mathbf i}(d)\text{ s.t. }\dist_{\calD}(f,g)\leq \alpha\}$.

Suppose we have a sequence of indices $1\leq i_1<i_2<\cdots<i_l\leq m$ denoted by $\mathbf i$ for short. Let $\calD_{\mathbf i}$ denote the conditional distribution of $\calD$ on $\bigcup\limits_{j=1}^lX_{i_j}$. A \emph{$(d,l,t,\epsilon)$ distance approximation oracle} is an algorithm taking a length-$l$ sequence $\mathbf i$ of indices and $f\in\{0,1\}^X$ as input, and performing $\ap_{\calD_{\mathbf i}}(f\ac,(\epsilon,0),d)$ on composition $\property^t$. In other words, this algorithm performs distance approximation on any given $l$-sub-union ($l$ is typically small) of the $m$ ground sets. For convenience of use, we require the success probability of the oracle to be at least $\frac{11}{12}$.

%such that $\forall \alpha$,:
%\begin{enumerate}
%\item $\forall f\in\widehat\property^{\mathbf i}_{\calD_{\mathbf i}}(d,\alpha)$, with probability at least $\frac{2}{3}$, it holds that $\hat\alpha\leq\alpha+\epsilon$;
%\item $\forall f\notin\widehat\property^{\mathbf i}_{\calD_{\mathbf i}}(d,\alpha_2)$, with probability at least $$.
%\end{enumerate}

%% file: additive.tex
\section{The Composition Lemma}
\label{sec:additive}

%\subsection{Bi-Criteria Distance Approximation}
\begin{lemma}[Composition Lemma]
\label{thm:additive}
Suppose $\property$ is the composition of $m$ additive properties defined in Section \ref{subsec:composition}. Let $\calD$ be a semi-uniform distribution. For parameters $\lambda>0,\alpha\in[0,1]$ and $\mu,\epsilon\in(0,1)$ taken as input, there exists $l=O(\frac{1}{\epsilon\mu^2}+\frac{1}{\epsilon^2})$ such that we have an algorithm that performs $\ap_{\calD}(f\ac,(\epsilon,\mu),\lambda m)$ by calling once a $((1+\frac{\mu}{2})\lambda l,l,\frac{4\lambda}{\epsilon},\frac{\epsilon}{2})$ distance approximation oracle. Suppose the query complexity and the unlabeled sample complexity of the oracle are $q$ and $N$, repectively. Then the query complexity and the unlabeled sample complexity of the algorithm are $q$ and $O(\frac{Nm}{l})$, respectively.

%$((1+\frac{\mu}{2})\lambda l,l,t,\alpha+\frac{\epsilon}{3},\alpha+\frac{\epsilon}{2})$-partial testing using at most $q(l,\lambda,\alpha,\mu,\epsilon)$ queries on at most $U(l,\lambda,\alpha,\mu,\epsilon)$ unlabeled samples with truncation $t=\frac{6\lambda}{\epsilon}$. Then, there is an active testing algorithm that distinguishes $f\in\property_{\calD}(\lambda m,\alpha)$ and $f\notin\property_{\calD}((1+\mu)\lambda m,\alpha+\epsilon)$ using at most $O(q(s,\lambda,\alpha,\mu,\epsilon))$ queries on at most $O(U(s,\lambda,\alpha,\mu,\epsilon)\cdot \frac{m}{s})$ unlabeled samples for $s=O(\frac{1}{\epsilon\mu^2}+\frac{1}{\epsilon^2})$.
\end{lemma}

\begin{proof}
The algorithm first picks indices $1\leq i_1<i_2<\cdots i_l\leq m$ uniformly at random for $l=O(\frac{1}{\epsilon\mu^2}+\frac{1}{\epsilon^2})$. Then the algorithm asks for $O(\frac{Nm}{l})$ unlabeled examples to make sure with probability at least $\frac{11}{12}$, there are at least $N$ examples lying in $\bigcup\limits_{j=1}^lX_{i_j}$. These examples can be treated as drawn independently at random according to $\calD_{\mathbf i}$, where $\mathbf i=(i_1,i_2,\cdots,i_l)$. Finally, the algorithm calls the oracle to approximate the distance from $f$ to $\property^{t}((1+\frac{\mu}{2})\lambda l)$ truncated by $t=\frac{4\lambda}{\epsilon}$ on distribution $\calD_{\mathbf i}$ up to an additive error $\frac{\epsilon}{2}$ using these unlabeled examples and outputs what the oracle outputs. 

The correctness of the algorithm follows from the following two lemmas and the Union Bound.
\end{proof}

\begin{lemma}
\label{lm:in}
Suppose $t=\frac{4\lambda}{m}$. If $f\in\property_{\calD}(\lambda m,\alpha)$, then choosing $l=O(\frac{1}{\epsilon\mu^2}+\frac{1}{\epsilon^2})$ is enough to make sure that with probability at least $\frac{5}{6}$, $f\in\property^{t}_{\calD_{\mathbf i}}((1+\frac{\mu}{2})\lambda l,\alpha+\frac{\epsilon}{2})$.
\end{lemma}

\begin{lemma}
\label{lm:notin}
Suppose $t=\frac{4\lambda}{m}$. If $f\notin\property_{\calD}((1+\mu)\lambda m,\alpha)$, then choosing $l=O(\frac{1}{\epsilon\mu^2}+\frac{1}{\epsilon^2})$ is enough to make sure that with probability at least $\frac{5}{6}$, $f\notin\property^{t}_{\calD_{\mathbf i}}((1+\frac{\mu}{2})\lambda l,\alpha-\frac{\epsilon}{2})$.
\end{lemma}

\begin{proof}[Proof of Lemma \ref{lm:in}]
By the choice of truncation $t=\frac{4\lambda}{\epsilon}$, according to Lemma \ref{lm:truncation}, we know $f\in\property^t_{\calD}(\lambda m,\alpha+\frac{\epsilon}{4})$. Suppose $\dist_{\calD}(f,g)\leq \alpha+\frac{\epsilon}{4}$ for some $g\in\property^t(\lambda m)$. According to the Multiplicative Chernoff Bound for sampling without replacement, choosing $l=O(\frac{1}{\epsilon\mu^2})$ is enough to make sure that with probability at least $\frac{11}{12}$, $\exists g'$ s.t.\ $g'\in\property^{t}((1+\frac{\mu}{2})\lambda l)$ and $\dist_{\calD_{\mathbf i}}(g,g')=0$.\footnote{$g'$ is chosen such that $g'|_{X_{i}}\in\calC_{i}^0$ for all $i\notin \{i_1,i_2,\cdots,i_l\}$ and $g'|_{X_{i}}=g|_{X_{i}}$ for all $i\in\{i_1,i_2,\cdots,i_l\}$. The fact that the $k_i$'s of $g$ are bounded between 0 and $t=\frac{4\lambda}{\epsilon}$ allows us to use the Multiplicative Chernoff Bound.} According to the Chernoff Bound for sampling without replacement, choosing $l=O(\frac{1}{\epsilon^2})$ is enough to make sure that with probability at least $\frac{11}{12}$, $\dist_{\calD_{\mathbf i}}(f,g)\leq\alpha+\frac{\epsilon}{2}$. By the Union Bound, these two events happen at the same time with probability at least $\frac{5}{6}$, and in this case, $f\in\property^{t}_{\calD_{\mathbf i}}((1+\frac{\mu}{2})\lambda l,\alpha+\frac{\epsilon}{2})$. 
\end{proof}

\begin{proof}[Proof of Lemma \ref{lm:notin}]
According to Lemma \ref{lm:truncation}, we know $f\notin\property^t_{\calD}((1+\mu)\lambda m,\alpha)$. Therefore, by definition, there exists $g\in\property^t((1+\mu)\lambda m)$ with the following two properties:\footnote{E.g., choose $g$ to be the closest or approximately-closest function in the class to $f$.  Note that $\property^t((1+\mu)\lambda m)$ can't be empty, because $\property^t((1+\mu)\lambda m)\supseteq\property^t(0)=\property(0)\neq\emptyset$.}
\begin{enumerate}
\item $\dist_{\calD}(f,g)>\alpha$;
\item $\forall g'\in \property^t((1+\mu)\lambda m),\dist_{\calD}(f,g')>\dist_{\calD}(f,g)-\frac{\epsilon}{4}\cdot\frac{l}{m}$.
\end{enumerate}
Suppose $g|_{X_i}\in\calC_{i}^{k_i}$ for $k_i\leq t=\frac{4\lambda}{\epsilon}$ satisfying $k:=\sum\limits_{i=1}^mk_i\leq(1+\mu)\lambda m$. We enlarge $k_i$ to $k'_i\in[k_i,t]$ to make sure that $k':=\sum\limits_{i=1}^mk'_i=(1+\mu)\lambda m$.\footnote{$k'_i$ doesn't have to be an integer. Also note that $mt=\frac{4\lambda}{\epsilon}\cdot m>4\lambda m>(1+\mu)\lambda m$.} According to the Multiplicative Chernoff Bound for sampling without replacement, choosing $l=O(\frac{1}{\epsilon\mu^2})$ is enough to make sure that with probability at least $\frac{11}{12}$, $\sum\limits_{j=1}^lk'_{i_j}\geq (1+\frac{\mu}{2})\lambda l$.

Now suppose it's the case that $\sum\limits_{j=1}^lk'_{i_j}\geq (1+\frac{\mu}{2})\lambda l$. Then, according to the second property of $g$, we know $$\forall g'\in\property^{t}((1+\frac{\mu}{2})\lambda l),\dist_{\calD_{\mathbf i}}(f,g')> \dist_{\calD_{\mathbf i}}(f,g)-\frac{\epsilon}{4}.$$ Otherwise, we can swap $g'$ for $g$ on $\bigcup\limits_{j=1}^lX_{i_j}$ causing a violation of the second property of $g$.

Finally, according to the Chernoff Bound for sampling without replacement, choosing $l=O(\frac{1}{\epsilon^2})$ is enough to make sure that with probability at least $\frac{11}{12}$, $\dist_{\calD_{\mathbf i}}(f,g)>\alpha-\frac{\epsilon}{4}$. Therefore, by the Union Bound, with probability at least $\frac{5}{6}$, $$\forall g'\in\property^{t}((1+\frac{\mu}{2})\lambda l),\dist_{\calD_{\mathbf i}}(f,g')> \dist_{\calD_{\mathbf i}}(f,g)-\frac{\epsilon}{4}>\alpha-\frac{\epsilon}{2},$$ a completion of the proof.
\end{proof}

%% file: interval.tex
\section{Distance Approximation for Unions of $d$ Intervals}
\label{sec:interval}
In this section, we consider the concept class $\dintervals$ of binary functions $f$ defined on $\mathbb{R}$ such that $f^{-1}(1)$ is a union of at most $d$ intervals. We use $\udi_{\calD}(f,\epsilon,d)$ to denote the distance approximation task $\da_{\calD}(f,\epsilon)$ when the underlying hypothesis class is the class $\interval(d)$ of unions of $d$ intervals. We show (Theorem \ref{thm:main}) that there is an $\udi_{\calD}(f,\epsilon,d)$ algorithm in the active model with query complexity independent of $d$ even when the data distribution is unknown to the algorithm.

\begin{theorem}[main theorem]
\label{thm:main}
Suppose $d>0$ and $\epsilon\in (0,\frac{1}{2})$ are given as input. Let $\mathcal{D}$ be an unknown distribution on $\mathbb{R}$. Suppose we have \emph{active} access to an input function $f:\mathbb{R}\rightarrow \{0,1\}$ with respect to $\calD$. There is an $\udi_{\calD}(f\ac,\epsilon,d)$ algorithm using $O(\frac{1}{\epsilon^6}\log\frac{1}{\epsilon})$ queries on $O(\frac{d}{\epsilon^2}\log\frac{1}{\epsilon})$ unlabeled examples.
\end{theorem}
Before proving Theorem \ref{thm:main}, we first introduce two helper results, Theorems \ref{thm:unlabeled} and \ref{thm:reductiontoquery}.

\subsection{Relationship between Query Testing and Active Testing}
\label{subsec:relationship}
As \citet{BBBY12} have pointed out, in the task of testing unions of $d$ intervals in the query testing framework, any known distribution can be reduced to uniform distribution on $[0,1]$ by its CDF. Our following theorem shows that once we can deal with arbitrary distributions for query testing, we can automatically deal with unknown distributions for active testing, improving a previous upper bound on unlabeled sample complexity in \citep{BBBY12}.

\begin{theorem}
\label{thm:unlabeled}
Let $\calC$ be a concept class on ground set $X$ with VC-dimension $d$. Suppose $\epsilon\in(0,\frac{1}{2})$. Suppose there is a $\pt_{\calD}(f\qu,\frac{\epsilon}{2})$ algorithm $\calA$ using at most $q$ queries on \emph{arbitrarily} given distribution $\calD$ with finite support. Suppose all the queries algorithm $\calA$ makes lie in the support of $\calD$.\footnote{For $\tot$ and $\da$, we can assume without loss of generality that the algorithm never queries examples outside the support of the distribution, but this is not without loss of generality for $\pt$, because $f\in\calC$ is stronger than $\exists g\in\calC,\dist_{\calD}(f,g)=0$.} Then, there is a $\pt_{\calD}(f\ac,\epsilon)$ algorithm $\calB$ using at most $O(q)$ queries on $O(\frac{d}{\epsilon}\log\frac{1}{\epsilon})$ unlabeled examples, even when distribution $\calD$ is \emph{unknown} to algorithm $\calB$.
\end{theorem}

\begin{proof}[Proof of Theorem \ref{thm:unlabeled}]
Algorithm $\calB$ first draws $N=O(\frac{d}{\epsilon}\log\frac{1}{\epsilon})$ unlabeled examples: $x_1,x_2,\cdots,x_N$. We use $\calS$ to denote the uniform distribution over these unlabeled examples. By VC Theory, we know if $\forall g\in\calC,\dist_{\calD}(f,g)>\epsilon$, then with probability at least $\frac{5}{6}$, $\forall g\in\calC,\dist_{\calS}(f,g)>\frac{\epsilon}{2}$. So algorithm $\calB$ only needs to call algorithm $\calA$ to distinguish $f\in\calC$ and $\forall g\in\calC,\dist_{\calS}(f,g)>\frac{\epsilon}{2}$ with probability at least $\frac{5}{6}$. By the Union Bound, algorithm $\calB$ succeeds with probability at least $\frac{2}{3}$.
\end{proof}

Therefore, since \citet{BBBY12} have an algorithm in the query testing framework that can distinguish $f\in\interval(d)$ and $f\notin\interval(d,\epsilon)$ using $O(\frac{1}{\epsilon^4})$ queries, there is an algorithm in the active testing framework that can distinguish $f\in\interval(d)$ and $f\notin\interval_{\calD}(d,\epsilon)$ using $O(\frac{1}{\epsilon^4})$ queries on $O(\frac{d}{\epsilon}\log\frac{1}{\epsilon})$ unlabeled examples, even when the distribution $\calD$ is unknown, according to Theorem \ref{thm:unlabeled}. Here, the unlabeled sample complexity is $O(\frac{d}{\epsilon}\log\frac{1}{\epsilon})$, an improvement from $O(\frac{d^2}{\epsilon^6})$ in their original paper.
The theorem can be easily generalized to distance approximation.

\begin{theorem}
\label{thm:reductiontoquery}
Let $\calC$ be a concept class on ground set $X$ with VC-dimension $d$. Suppose $\epsilon\in (0,\frac{1}{2})$. Suppose there is a $\da_{\calD}(f\qu,\frac{\epsilon}{2},\calD)$ algorithm $\calA$ using at most $q$ queries on \emph{arbitrarily} given distribution $\calD$ with finite support. Then, there is a $\da_{\calD}(f\ac,\epsilon)$ algorithm $\calB$ using at most $O(q)$ queries on $O(\frac{d}{\epsilon^2}\log\frac{1}{\epsilon})$ unlabeled examples, even when distribution $\calD$ is \emph{unknown} to algorithm $\calB$.
\end{theorem}

%The unlabeled sample complexity of the algorithm is improved to $O(\frac{d}{\epsilon}\log\frac{1}{\epsilon})$ by our Theorem \ref{thm:unlabeled}.

%In the active tolerant testing model, given a distribution $\mathcal{D}$, to test the concept class $\dintervals$ is to distinguish $f\in\interval_\mathcal{D}(d,\alpha_1)$ and $f\notin\interval_\mathcal{D}(d,\alpha_2)$ for any given parameters $\alpha_1<\alpha_2$, where $\interval_\mathcal{D}(d,\alpha)$ is the class of functions $f$ satisfying $\exists g\in \mathcal{I}(d)$ s.t. $\dist_{\mathcal{D}}(f,g)\leq \alpha$. When $\calD$ is the uniform distribution on unit interval $[0,1]$, we omit it for short.

\subsection{Proof of Theorem \ref{thm:main}}
By Theorem \ref{thm:reductiontoquery}, we only need to show a distance approximation algorithm with an additive error bounded below $\epsilon$ using $O(\frac{1}{\epsilon^6}\log\frac{1}{\epsilon})$ queries on arbitrary \emph{known} distribution $\calD$. As pointed out by \citet{BBBY12}, any known distribution can be reduced to uniform distribution on $[0,1]$ by its CDF. Therefore, we only consider $\calD$ as the uniform distribution on $[0,1]$ in this section and we omit it for simplicity.

We first reveal a basic property of unions of $d$ intervals.
\begin{lemma}
\label{lm:reductiontobicriteria}
$\forall \epsilon\in(0,\frac{1}{2}),\forall \alpha\in[0,1],\forall d>\frac{2}{\epsilon},\interval((1+\frac{\epsilon}{2})d,\alpha-\epsilon)\subseteq\interval(d,\alpha)$.
\end{lemma}
\begin{proof}
$\forall f\in\interval((1+\frac{\epsilon}{2})d,\alpha-\epsilon),\exists g\in\interval((1+\frac{\epsilon}{2})d)$ s.t. $\dist(f,g)\leq \alpha-\epsilon$. Assume $g$ uses $k\leq (1+\frac{\epsilon}{2})d$ intervals. Without loss of generality, we can assume that $k\geq d$. We remove $\lceil\frac{\epsilon k}{2}\rceil$ shortest intervals from the $k$ intervals. The number of remaining intervals is at most $(1-\frac{\epsilon}{2})k\leq (1-\frac{\epsilon}{2})(1+\frac{\epsilon}{2})d\leq d$. The distance increase is upper bounded by $(\frac{\epsilon k}{2}+1)\cdot\frac{1}{k}\leq \frac{\epsilon}{2}+\frac{1}{d}\leq\epsilon$.
\end{proof}

%\begin{proof}[Proof of Theorem \ref{thm:main}]
Now we come back and prove Theorem \ref{thm:main}.
If $d\leq\frac{8}{\epsilon}$, we can simply do agnostic learning using $O(\frac{d}{\epsilon^2}\log\frac{1}{\epsilon})=O(\frac{1}{\epsilon^3}\log\frac{1}{\epsilon})$ queries and unlabeled examples. So in the rest of the proof, we assume $d>\frac{8}{\epsilon}$. We pick the largest positive integer $m$ satisfying $m\leq\frac{\epsilon d}{8}$ and we define $\lambda=\frac{d}{m}=O(\frac{1}{\epsilon})$.

Since the data distribution is assumed uniform on $[0,1]$, we can assume without loss of generality that our ground set $X$ is $[0,1]$ and $f\in\{0,1\}^X$. We evenly cut $X$ into $m$ pieces: $X_1=[0,\frac{1}{m}],X_2=(\frac{1}{m},\frac{2}{m}],X_3=(\frac{2}{m},\frac{3}{m}],\cdots,X_m=(\frac{m-1}{m},1]$. $\forall 1\leq i\leq m,\forall k\in\mathbb{N}$, we define $\calC_i^k$ to be the class of binary functions $f$ on $X_i$ such that $f^{-1}(1)$ is a union of at most $k$ intervals. Note that $\calC_i^0\neq\emptyset$. Therefore, we can define $\property$, the composition of $m$ additive properties as in Section \ref{subsec:composition}. 

Note that for any $d'>0$ and any truncation $t>0$, the concept class $\property^{t}(d')$ has VC-dimension at most $2d'$. Therefore, simply by VC Theory for agnostic learning, for any $\mu,\epsilon'\in(0,\frac{1}{2}),l=O(\frac{1}{\epsilon'\mu^2}+\frac{1}{\epsilon'^2})$, we have a $((1+\frac{\mu}{2})(1+\frac{\epsilon}{8})\lambda l,l,\frac{4(1+\frac{\epsilon}{8})\lambda}{\epsilon'},\frac{\epsilon'}{2})$ distance approximation oracle using $O(\frac{(1+\frac{\mu}{2})(1+\frac{\epsilon}{8})\lambda l}{(\frac{\epsilon'}{2})^2}\log\frac{1}{\frac{\epsilon'}{2}})=O(\frac{l}{\epsilon'^2\epsilon}\log\frac{1}{\epsilon'})=O((\frac{1}{\epsilon'^3\epsilon\mu^2}+\frac{1}{\epsilon'^4\epsilon})\log\frac{1}{\epsilon'})$ queries and unlabeled examples. By the Composition Lemma (Lemma \ref{thm:additive}), we have an algorithm that outputs $\hat\alpha$ such that $\forall \alpha$,
\begin{enumerate}
\item $\forall f\in\property((1+\frac{\epsilon}{8})\lambda m,\alpha)$, it holds with probability at least $\frac{2}{3}$ that $\hat\alpha\leq\alpha+\epsilon'$;
\item $\forall f\notin\property((1+\mu)(1+\frac{\epsilon}{8})\lambda m,\alpha)$, it holds with probability at least $\frac{2}{3}$ that $\hat\alpha>\alpha-\epsilon'$.
\end{enumerate}

Choose $1+\mu=\frac{1+\frac{\epsilon}{4}}{1+\frac{\epsilon}{8}}$ and note that $\lambda m=d,\interval(d,\alpha)\subseteq\property(d+m,\alpha)\subseteq\property((1+\frac{\epsilon}{8})d,\alpha)$ and $\property((1+\frac{\epsilon}{4})d,\alpha)\subseteq \interval((1+\frac{\epsilon}{4})d,\alpha)$, we have $\forall \alpha$,
\begin{enumerate}
\item $\forall f\in\interval(d,\alpha)$, it holds with probability at least $\frac{2}{3}$ that $\hat\alpha\leq\alpha+\epsilon'$;
\item $\forall f\notin\interval((1+\frac{\epsilon}{4})d,\alpha)$, it holds with probability at least $\frac{2}{3}$ that $\hat\alpha>\alpha-\epsilon'$.
\end{enumerate}

This is an $(\epsilon',\frac{\epsilon}{4})$-bi-criteria tester for unions of $d$ intervals. According to the Composition Lemma (Lemma \ref{thm:additive}), the query complexity and the unlabeled sample complexity of the algorithm are $O((\frac{1}{\epsilon'^3\epsilon\mu^2}+\frac{1}{\epsilon'^4\epsilon})\log\frac{1}{\epsilon'})=O((\frac{1}{\epsilon'^3\epsilon^3}+\frac{1}{\epsilon'^4\epsilon})\log\frac{1}{\epsilon'})$ and $O((\frac{l}{\epsilon'^2\epsilon}\log\frac{1}{\epsilon'})\cdot\frac{m}{l})=O(\frac{d}{\epsilon'^2}\log\frac{1}{\epsilon'})$.

Now we define $\epsilon'=\frac{\epsilon}{2}$ and by rewriting the second statement in an equivalent way, we get $\forall \alpha$,
\begin{enumerate}
\item $\forall f\in\interval(d,\alpha)$, it holds with probability at least $\frac{2}{3}$ that $\hat\alpha\leq\alpha+\frac{\epsilon}{2}<\alpha+\epsilon$;
\item $\forall f\notin\interval((1+\frac{\epsilon}{4})d,\alpha-\frac{\epsilon}{2})$, it holds with probability at least $\frac{2}{3}$ that $\hat\alpha>(\alpha-\frac{\epsilon}{2})-\frac{\epsilon}{2}=\alpha-\epsilon$.
\end{enumerate}

Finally, by Lemma \ref{lm:reductiontobicriteria}, we have $\interval((1+\frac{\epsilon}{4})d,\alpha-\frac{\epsilon}{2})\subseteq \interval(d,\alpha)$, which completes the proof.

%% file: neighbor.tex
\section{Estimating the Performance of $k$-Nearest Neighbor Algorithms}
\label{sec:neighbor}
%\subsection{The $k$-Nearest Neighbor Algorithm}
In this section, we develop testers for estimating the performance of $k$-Nearest Neighbor ($\knn$) algorithms \citep{FH51,FH89,CH67}. 

Let $\calD$ be a distribution on a ground set $X$. Suppose that every point $x\in X$ has a (true) label $f(x)\in\{0,1\}$. In addition, we have a distance metric $\D :X\times X\rightarrow \mathbb{R}_{\geq 0}$ that is symmetric, nonnegative and satisfies the triangle inequality. The \emph{$k$-Nearest Neighbor algorithm with soft predictions} ($\knnsoft$) is given a pool $S$ of unlabeled examples, sampled iid from $\calD$, and for any input $x\in X$, finds its $k$ nearest examples $x_1,x_2,\cdots, x_k\in S$ with respect to the distance metric $\D$ and outputs $\hat{f}(x)=\frac{1}{k}\sum\limits_{i=1}^kf(x_i)$ as an approximation of $f(x)$. In this paper, we assume the $k$ nearest examples are calculated by an oracle $M$, i.e., when given $x$ and $S$, $M$ calculates the $k$ nearest examples to $x$ in $S$. There may be ties when distances to $x$ are compared and we assume $M$ breaks ties according to some (probably random) mechanism.

The \emph{$k$-Nearest Neighbor algorithm with hard predictions} ($\mv$) does the same thing as $\knnsoft$, except that $\hat{f}(x)$ is chosen as the majority vote $I[\frac{1}{k}\sum\limits_{i=1}^kf(x_i)>0.5]$.\footnote{$I[\cdot]$ is the indicator function of a statement, which takes value 1 if the statement is true and value 0 if the statement is false.}

For both algorithms, we use $\err(x)=|\hat f(x)-f(x)|$ to denote the $L^1$ error on point $x\in X$.  For soft prediction, we will penalize the algorithm by taking the $p$th power of the $L^1$ error for positive integer $p$.

\subsection{Estimating the Performance of $\knnsoft$}
\label{subsec:knnsoft}
Given a loss function $\loss(\cdot)$, we can measure the performance of $\knnsoft$ by its expected loss $\mathbb E_{x}[\loss(\err(x))]$. The expectation is over the random draw of $x$ with respect to distribution $\calD$ and the randomness of the oracle $M$ when ties occur. In this paper, we focus on the $p$th-power loss $\mathbb E_{x}[(\err(x))^p]$ for positive integer $p$. Let $\tsoft_{\calD} (f,\epsilon,S,k)$ denote the testing task of approximating the expected loss of a $\knnsoft$ algorithm up to an additive error $\epsilon$ with success probability at least $\frac{2}{3}$. We consider the testing task in the active model, in which the tester is only allowed to query labels of examples in an unlabeled pool sampled iid from $\calD$. In addition to the given unlabeled pool $S$ from which $\knnsoft$ would learn, we allow the $\tsoft_{\calD} (f\ac,\epsilon,S,k)$ tester to sample fresh unlabeled examples and query their labels. We assume the tester has access to the oracle $M$.%The task of $\knn(f,\epsilon,N)$ is to approximate the expected loss of the $\knn$ algorithm when the true labels are $f$ and the training set has size $N$ within additive error $\epsilon$. When showing upper bound results, we assume the tie breaking mechanism $M$ is given as input and when showing lower bound results, we assume the $\knn(f,\epsilon,N)$ algorithm can assume any $M$ it likes.
\begin{theorem}
\label{thm:pth}
Suppose we consider the $p$th-power loss for $p\in\mathbb N^*$. There is a tester $\tsoft_{\calD}(f\ac,\epsilon,S,k)$ using $O(\frac{p}{\epsilon^2})$ queries on $N+O(\frac{1}{\epsilon^2})$ unlabeled examples when the unlabeled pool $S$ has size $N$. The underlying distribution $\calD$ is assumed unknown to the tester. Moreover, the tester has success probability at least $\frac{2}{3}$ for \emph{any} unlabeled pool $S$.
\end{theorem}

Before proving the theorem, we first show a simple tester that works for any loss function $\mathrm{loss}(\cdot)$ bounded in $[0,1]$ with $L$-Lipschitz property\footnote{We say $\mathrm{loss}(\cdot)$ has $L$-Lipschitz property if $\forall x_1,x_2\in[0,1],|\loss(x_1)-\loss(x_2)|\leq L|x_1-x_2|$.} using $O(\frac{L^2}{\epsilon^4}\cdot\log\frac{1}{\epsilon})$ queries on $N+O(\frac{1}{\epsilon^2})$ unlabeled examples. The tester runs for $O(\frac{1}{\epsilon^2})$ iterations and in each $i$th iteration, the tester samples a fresh unlabeled example $x$ and then queries the labels of $w=O(\frac{L^2}{\epsilon^2}\log\frac{1}{\epsilon})$ examples $x_1,x_2,\cdots,x_w$ sampled independently at random uniformly from the $k$ nearest neighbors of $x$ in $S$. The estimator for this iteration is $E_i=\mathrm{loss}(|\frac{1}{w}\sum\limits_{j=1}^wf(x_j)-f(x)|)$. The final output of the tester is the average of all $E_i$'s for all iterations $i$.

We prove Theorem \ref{thm:pth} by slightly modifying the above tester's each iteration for $p$th-power loss. Instead of looking at the labels of $w$ examples, we only need to look at $p$ labels of $x_1,x_2,\cdots, x_p$, still sampled independently at random uniformly from the $k$ nearest neighbors of $x$ in $S$. In this case, $E_i$ is defined to be $\prod\limits_{j=1}^p|f(x_j)-f(x)|$. The final output of the tester is still the average of $E_i$'s.

%The proof of Theorem \ref{thm:pth} is based on the following lemma.

%\begin{lemma}
%\label{lm:pth}
%Suppose the examples in the trainning set $T$ and the test point $x$ are sampled independently at random according to $\calD$. Suppose $x_1,x_2,\cdots,x_p$ are sampled independently at random uniformly from the $k$ nearest neighbors of $x$ in $T$. Suppose $e_i$ is defined to be $|f(x_i)-f(x)|$Then, $\mathbb E_{T,x}[\mathbb E_{x_1}[e_i]^p]=\mathbb E_{T,x,x_1,x_2,\cdots,x_p}[\prod\limits_{j=1}^pe_i]$.
%\end{lemma}

%\begin{proof}
%\begin{equation}
%\begin{split}
%&\mathbb E_{T,x}[|\mathbb E_{x_1}[f(x_1)-f(x)]|^p]\\
%=&\mathbb E_{T,x}[|(\mathbb E_{x_1}[f(x_1)-f(x)])^p|]\\
%=&\mathbb E_{T,x}[|\underbrace{(\mathbb E_{x_1}[f(x_1)-f(x)])(\mathbb E_{x_1}[f(x_1)-f(x)])\cdots (\mathbb E_{x_1}[f(x_1)-f(x)])}_{p}|]\\
%=&\mathbb E_{T,x}[|\prod\limits_{j=1}^p\mathbb E_{x_j}[f(x_j)-f(x)]|]\\
%=&\mathbb E_{T,x}[|\mathbb E_{x_1,x_2,\cdots,x_p}[\prod\limits_{j=1}^p(f(x_j)-f(x))]|]\\
%=&\mathbb E_{T,x}[\mathbb E_{x_1,x_2,\cdots,x_p}[\prod\limits_{j=1}^p|f(x_j)-f(x)|]]\\
%=&\mathbb E_{T,x,x_1,x_2,\cdots,x_p}[\prod\limits_{j=1}^p|f(x_j)-f(x)|]
%\end{split}
%\end{equation}

%\end{proof}

\begin{proof}[Proof of Theorem \ref{thm:pth}]
We use $e_j$ to denote $|f(x_j)-f(x)|$. To show the above tester works, we first look at the value we want to estimate: $\mathbb E_{x}[(\err(x))^p]=\mathbb E_{x}[(\mathbb E_{x_1}[e_1])^p]$, where $x_1$ is sampled uniformly from the $k$ nearest neighbors of $x$ in $T$. Note that $x_1,x_2,\cdots,x_p$ are iid, so we know $\mathbb E_{x}[ (\mathbb E_{x_1}[e_1])^p]=\mathbb E_{x}[\mathbb E_{x_1,x_2,\cdots,x_p}[e_1e_2\cdots e_p]]=\mathbb E_{x,x_1,x_2,\cdots,x_p}[\prod\limits_{j=1}^p|f(x_j)-f(x)|]$. The Chernoff Bound thus completes the proof.
\end{proof}

Theorem \ref{thm:pth} also holds naturally for Weighted Nearest Neighbor algorithms \citep{R66} with soft predictions, in which $\hat f(x)$ is a weighted average of $f(x')$ for all $x'\in S$ where the weights depend on the distances $\D(x',x)$, simply by sampling $x_1,x_2,\cdots,x_p$ iid from $S$ according to the weights. 

In Theorem \ref{thm:pthpowerlower} (Section \ref{subsec:lower}), we will show a matching lower bound for the $O(\frac{p}{\epsilon^2})$ query complexity.
\subsection{Finding an Approximately-Best Choice of $k$}
\label{subsec:bestk}
Based on the result in Section \ref{subsec:knnsoft}, we are able to construct an algorithm that approximately optimizes the choice of $k$ in the $\knnsoft$ algorithm.

Suppose we have active access to the true label $f$ with respect to distribution $\calD$ over ground set $X$ with distance metric $\D$. Suppose the size of the unlabeled pool $S$ is fixed to be $N$. We use $\mathrm{loss}_k$ to denote the expected loss of the $\knnsoft$ algorithm and consider how the $\knnsoft$ algorithm performs with different values of $k$. We assume the oracle $M$ uses the same tie-breaking mechanism for different values of $k$. Specifically, given $x$ and $S$, $M$ arranges the examples in $S$ as $x_1,x_2,\cdots,x_N$ so that $\forall i,\D(x_i,x)\leq \D(x_{i+1},x)$. $x_1,x_2,\cdots,x_k$ are taken by $\knnsoft$ as the $k$ nearest neighbors of $x$ for any $k\in\{1,2,\cdots,N\}$.

\begin{lemma}
\label{lm:bestk}
Suppose $k_1\leq k_2$ and the loss function $\loss(\cdot)$ is $L$-Lipschitz. Then, $|\loss_{k_1}-\loss_{k_2}|\leq L\cdot(1-\frac{k_1}{k_2})$.
\end{lemma}
\begin{proof}
When the test point $x$ is chosen, we use $x_1,x_2,\cdots,x_{k_2}$ to denote the closest $k_2$ points to $x$ in $S$, arranged in non-decreasing order of their distances to $x$. Each $x_i$ might be random because ties might be broken randomly. We use $e_i$ to denote $|f(x_i)-f(x)|$. Note that we have $\loss_{k_1}=\mathbb E_{x,x_1,x_2,\cdots,x_{k_1}}[\loss(\frac{1}{k_1}\sum\limits_{i=1}^{k_1}e_i)]$ and $\loss_{k_2}=\mathbb E_{x,x_1,x_2,\cdots,x_{k_2}}[\loss(\frac{1}{k_2}\sum\limits_{i=1}^{k_2}e_i)]$. Therefore,
\begin{equation}
\begin{split}
&|\loss_{k_1}-\loss_{k_2}|\\
\leq &\mathbb E_{x,x_1,x_2,\cdots,x_{k_2}}[|\loss(\frac{1}{k_1}\sum\limits_{i=1}^{k_1}e_i)-\loss(\frac{1}{k_2}\sum\limits_{i=1}^{k_2}e_i)|]\\
\leq &L\cdot \mathbb E_{x,x_1,x_2,\cdots,x_{k_2}}[|\frac{1}{k_1}\sum\limits_{i=1}^{k_1}e_i-\frac{1}{k_2}\sum\limits_{i=1}^{k_2}e_i|]\\
= &L\cdot \mathbb E_{x,x_1,x_2,\cdots,x_{k_2}}[|(\frac{1}{k_1}-\frac{1}{k_2})\sum\limits_{i=1}^{k_1}e_i-\frac{1}{k_2}\sum\limits_{i=k_1+1}^{k_2}e_i|]\\
\leq &L\cdot \mathbb E_{x,x_1,x_2,\cdots,x_{k_2}}[\max\{(\frac{1}{k_1}-\frac{1}{k_2})\sum\limits_{i=1}^{k_1}e_i,\frac{1}{k_2}\sum\limits_{i=k_1+1}^{k_2}e_i\}]\\
\leq &L\cdot\max\{(\frac{1}{k_1}-\frac{1}{k_2})\cdot k_1,\frac{1}{k_2}\cdot(k_2-k_1)\}\\
=&L\cdot(1-\frac{k_1}{k_2})
\end{split}
\end{equation}
\end{proof}
We say $k$ is $\epsilon$-approximately-best, if $\forall k'\in\{1,2,\cdots, N\},\mathrm{loss}_{k'}\geq \mathrm{loss}_{k}-\epsilon$. The following theorem states that we can find an $\epsilon$-approximately-best $k$ using a small number of queries.
\begin{theorem}
\label{thm:bestk}
Suppose $\knnsoft$ algorithms with an unlabeled pool $S$ of size $N$ are measured by $p$th-power loss for $p\in\mathbb N^*$. Suppose $\epsilon\in(0,\frac{1}{2})$. There is an algorithm that finds an $\epsilon$-approximately-best $k$ w.p.\ at least $\frac{2}{3}$ using $O(\frac{p^2\log N}{\epsilon^3}(\log\log N+\log p+\log\frac{1}{\epsilon}))$ queries on $N+O(\frac{p\log N}{\epsilon^3}(\log\log N+\log p+\log\frac{1}{\epsilon}))$ unlabeled examples.
\end{theorem}
\begin{proof}
If we apply Lemma \ref{lm:bestk} to $p$th-power loss, which is $p$-Lipschitz, we know for any $1\leq \frac{k_2}{k_1}\leq \frac{p}{p-\epsilon}$, it holds that $|\loss_{k_1}-\loss_{k_2}|\leq\epsilon$. If we define $t=\lfloor\log_{\frac{p}{p-\frac{\epsilon}{3}}}N\rfloor,k_{2i}=\lfloor(\frac{p}{p-\frac{\epsilon}{3}})^i\rfloor,k_{2i+1}=\lceil(\frac{p}{p-\frac{\epsilon}{3}})^i\rceil$ for $i=0,1,2,\cdots,t$, then we know $\exists 0\leq i\leq 2t+1$ such that $k_i$ is $\frac{\epsilon}{3}$-approximately-best. By Theorem \ref{thm:pth}, we can approximate $\loss_{k_i}$ for every $0\leq i\leq 2t+1$ up to an additive error $\frac{\epsilon}{3}$ using $O(\frac{pt\log t}{\epsilon^2})$ queries on $N+O(\frac{t\log t}{\epsilon^2})$ unlabeled examples.\footnote{Repeat the tester $O(\log t)$ times and take the median to boost its success probability to $1-O(\frac{1}{t})$.} The $k_i$ yielding the smallest approximation of $\loss_{k_i}$ is $\epsilon$-approximately-best. Note that $t=O(\frac{p\log N}{\epsilon})$, so the query complexity is $O(\frac{p^2\log N}{\epsilon^3}(\log\log N+\log p+\log\frac{1}{\epsilon}))$ and the unlabeled sample complexity is $N+O(\frac{p\log N}{\epsilon^3}(\log\log N+\log p+\log\frac{1}{\epsilon}))$.
\end{proof}
\subsection{Estimating the Performance of $\mv$}
\label{subsec:mv}
The performance of $\mv$ is naturally measured by its error rate $\mathbb E_{x}[\err(x)]$ and we use $\thard_{\calD}(f,\epsilon,S,k)$ to denote the corresponding testing task of estimating the error rate of $\mv$ up to an additive error $\epsilon$ with success probability at least $\frac{2}{3}$.

A trivial tester achieving this goal using $O(\frac{k}{\epsilon^2})$ queries on $N+O(\frac{1}{\epsilon^2})$ unlabeled examples is to use the empirical mean of $\err(x)$ as an estimator of $\mathbb E_{x}[\err(x)]$. %The algorithm runs $O(\frac{1}{\epsilon^2})$ iterations and in each iteration, the algorithm samples a fresh unlabeled example $x$ from $\calD$ and then uses $k+1$ queries to determine whether $\mv$ makes the correct prediction on $x$. The algorithm outputs the empirical mean of the errors of the $O(\frac{1}{\epsilon^2})$ iterations as an estimate of the error rate of $\mv$. 
This tester is not satisfactory because its query complexity grows with respect to $k$. In Section \ref{subsec:lower}, we will show (Theorem \ref{thm:mv}) that this linear growth with respect to $k$ can't be eliminated. Also, we will show (Theorem \ref{thm:reductionfrombandit}) that the $O(\frac{k}{\epsilon^2})$ query complexity is optimal if we assume a natural algorithm for \emph{approximating the fraction of good arms} ($\aga$) in the stochastic multi-armed bandit setting has the optimal query complexity. Before we show our lower bound results, we first define the problems of \emph{counting and approximating the number of good arms}.
\subsection{Counting and Approximating the Number of Good Arms}
\label{subsec:arm}
To show query complexity lower bound results for estimating the performance of $k$-Nearest Neighbor algorithms, we show reductions from two related problems in the stochastic multi-armed bandit setting: counting the number of good arms ($\cga$) and approximating the number of good arms ($\aga$).

The setting of stochastic multi-armed bandit problems \citep{R85} is as follows. The algorithm is given $n$ arms, denoted by $\mathbf A=(A_1,A_2,\cdots,A_n)$. Each arm is a distribution over $\mathbb R$ unknown to the algorithm. The algorithm adaptively accesses these arms to receive values independently sampled according to the distributions.

In this paper, we only consider arms with Bernoulli distributions. When given $\gamma\in(0,\frac{1}{2}]$, we define good arms to be arms with mean at least $\frac{1}{2}+\gamma$ and bad arms to be arms with mean at most $\frac{1}{2}-\gamma$.

The problem of $\cga(\mathbf A,\gamma)$ is, when given $\mathbf A$ in which every $A_i$ is either good or bad, to output the number of good arms among the given $n$ arms. The algorithm should output the correct answer with probability at least $\frac{2}{3}$.

The problem of $\aga(\mathbf A,\gamma,\epsilon)$ is a similar task to $\cga(\mathbf A,\gamma)$, except that we only need to approximate the correct answer up to an additive error $\epsilon n$.

The following lemma is developed by \citet{KCG16} as a useful tool for proving lower bounds in the stochastic multi-armed bandit setting.
\begin{lemma}[Change of measure]
\label{lm:changeofdistribution}
Suppose $\mathbf A=(A_1,A_2,\cdots,A_n)$ and $\mathbf A'=(A_1',A_2',\cdots,A_n')$ are two sequences of arms. Suppose an algorithm $\calA$ taking $n$ arms as input almost-surely terminates within finite time. Suppose $\mathcal E$ is an event defined in the probability space induced by the randomness of the arms and the internal randomness of algorithm $\calA$. Suppose $\tau_i$ is the number of queries on $A_i$ made by the algorithm. Then,
$$\sum\limits_{i=1}^n\mathbb E_{\calA,\mathbf A}[\tau_i]\mathrm{KL}(A_i,A_i')\geq D(\Pr_{\calA,\mathbf A}[\calE],\Pr_{\calA,\mathbf A'}[\calE]).\footnote{$\mathrm{KL}(X,Y)$ denotes the Kullback-Leibler divergence from distribution $Y$ to distribution $X$. If the two distributions $X$ and $Y$ are Bernoulli with means $x$ and $y$, their Kullback-Leibler divergence is the relative entropy $D(x,y)=x\log\frac{x}{y}+(1-x)\log\frac{1-x}{1-y}$.}$$
\end{lemma}

A simple special case ($n=1$) of the lemma is that to distinguish a coin with mean $\mu_1$ from a coin with mean $\mu_2$ with success probability at least $1-\delta$, an algorithm needs at least $\frac{D(1-\delta,\delta)}{D(\mu_1,\mu_2)}=\Omega(\frac{1}{D(\mu_1,\mu_2)}\log\frac{1}{\delta})$ queries in expectation for $\mu_1\neq \mu_2$ and $0<\delta\leq\frac{2}{5}$.
%Before introducing the reduction, we first describe the setting for the problem of approximating the fraction of good arms. Suppose we are given $n$ arms $A_1,A_2,\cdots, A_n$, each $A_i$ as a 0-1 coin with mean $\mu_i$. We are also given a parameter $\gamma\in(0,\frac{1}{2}]$. We assume that the $n$ arms are either ``good'', i.e., $\mu_i\geq \frac{1}{2}+\gamma$, or ``bad'', i.e., $\mu_i\leq\frac{1}{2}-\gamma$. The algorithm can choose to query each arm for multiple times to receive the results of independent coin tosses. The goal of the algorithm is to approximate the fraction of good arms among the $n$ arms within additive error $\epsilon$ with success probability at least $\frac{2}{3}$.

\subsection{Lower Bound Results}
\label{subsec:lower}
Our lower bound results in this section are stronger in the sense that the tester has query access to $f$, knows the distribution to be the uniform distribution $\calU$ over a finite ground set $X$ and is only supposed to work on some fixed tie-breaking mechanism. Moreover, we don't require the tester to have success probability at least $\frac{2}{3}$ for \emph{any} $S$; instead, the success probability is calculated over the random draw of $S$ and the internal randomness of the tester.

\begin{theorem}
\label{thm:pthpowerlower}Let $\calU$ be the uniform distribution over a finite ground set $X$. There exists a positive constant $c$ such that for any fixed $p\geq 1$, $\epsilon\in(0,\frac{1}{6\sqrt{e}})$ and oracle $M$ using any fixed tie-breaking mechanism, $\tsoft_{\calU}(f\qu,\epsilon,S,k)$ for $p$th-power loss requires at least $c\cdot\frac{p}{\epsilon^2}$ queries in the worst case over all finite metric spaces $(X,\D)$.
\end{theorem}
\begin{proof}
We define $\epsilon'=6\sqrt{e}\epsilon$. Note that $D(\frac{1-\epsilon'}{2p},\frac{1}{2p})=O(\frac{{\epsilon'}^2}{p})$ for $p\geq 1$ and $\epsilon'\in(0,1)$. Therefore, we only need to show that a  $\tsoft_{\calU}(f\qu,\epsilon,S,k)$ tester implies an algorithm that distinguishes a coin of mean $\frac{1-\epsilon'}{2p}$ from a coin of mean $\frac{1}{2p}$ with success probability at least $\frac{3}{5}$ using at most the same number of queries. We construct the algorithm in the following way.

The algorithm first constructs a $\knnsoft$ instance with ground set $X$ and distance metric $\D$. We first choose $k=\lceil\frac{c'p^2}{\epsilon^2}\rceil, b=\lceil\frac{6}{\epsilon}\rceil,N=\lceil c''\cdot (1+b)k\rceil$ and $m=\lceil \frac{c'''N^2}{1+b}\rceil\geq\frac{6N}{(1+b)\epsilon}$. Here, $c',c''$ and $c'''$ are sufficiently large constants. $X$ consists of a star with $m$ centers and $bm$ leaves. Each center $C$ has a distance $\D_C\in(1,2)$ to every leaf in the star and different centers have different values of $\D_C$ to avoid ties. The distance between each pair of leaves is 2 and the distance between each pair of centers is 1. 

The algorithm then simulates the tester $\tsoft_{\calU}(f\qu,\epsilon,S,k)$ on this $\knnsoft$ instance without knowing $f$ beforehand. Every time the tester queries the label of a new example, it simulates the result as follows. If the example being queried is a leaf, the result is 1. If the example being queried is a center, the result is obtained to be the same result of an independent toss of the coin we want to distinguish. Finally, if the output of $\tsoft_{\calU}(f\qu,\epsilon,S,k)$ is above $\frac{1}{2}[(1-\frac{1}{2p})^p+(1-\frac{1-\epsilon'}{2p})^p]$, the algorithm then guesses the coin to have mean $\frac{1-\epsilon'}{2p}$. Otherwise, the algorithm guesses the coin to have mean $\frac{1}{2p}$.

Now we show that the above algorithm correctly distinguishes the coins with success probability at least $\frac{3}{5}$. The process of the algorithm, by interchanging the randomness of the labels (coin tosses) and the internal randomness of the $\tsoft_{\calU}(f\qu,\epsilon,S,k)$ tester, can be viewed in the way that the true labels $f$ are determined before we run the $\tsoft_{\calU}(f\qu,\epsilon,S,k)$ tester. The leaves all have label 1 and each center is independently labeled 0 or 1 according to the result of a toss of the coin. After the labels $f$ are decided, the $\tsoft_{\calU}(f\qu,\epsilon,S,k)$ tester is then simulated  to approximate the $p$th-power loss of the $\knnsoft$ instance up to additive error $\epsilon$ with success probability at least $\frac{2}{3}$. 

Suppose the coin to be distinguished has mean $\mu$. Note that the total number of points in the ground set is $(1+b)m=\Omega(N^2)$, therefore we can make sure with probability at least $1-\frac{1}{40}$ that no two unlabeled examples lie on the same point. Because each random example has probability $\frac{1}{1+b}$ to lie in the centers and $N\geq c''\cdot (1+b)k$, therefore by choosing a sufficiently large $c''$, we can make sure with probability at least $1-\frac{1}{40}$ that in the unlabeled sample pool, there are at least $k$ examples lying at the centers. These two events happen at the same time with probability at least $1-\frac{1}{20}$ by the Union Bound. Conditioned on these two events happening, by a sufficiently large choice of $c'$, among those unlabeled examples lying at the centers, we can make sure that with probability at least $1-\frac{1}{20}$, the average of the labels of the $k$ examples with smallest $\D_C$ is contained in $(\mu-\frac{\epsilon}{6p},\mu+\frac{\epsilon}{6p})$. All these events happen at the same time with probability at least $(1-\frac{1}{20})^2\geq 1-\frac{1}{10}$, and in this case, every leaf outside the unlabeled pool $S$ has $L^1$ error in $(1-\mu-\frac{\epsilon}{6p},1-\mu+\frac{\epsilon}{6p})$ and thus has $p$th-power loss in $((1-\mu)^p-\frac{\epsilon}{6},(1-\mu)^p+\frac{\epsilon}{6})$. The total number of leaves in the unlabeled pool $S$ and centers is upper bounded by the size $N$ of the pool plus $m$, which contributes only a $\frac{N+m}{(b+1)m}\leq\frac{\epsilon}{3}$ fraction of the total number of points. Therefore, with probability at least $1-\frac{1}{10}$, the average $p$th-power loss of all points is contained in $((1-\mu)^p-\frac{\epsilon}{2},(1-\mu)^p+\frac{\epsilon}{2})$. %Finally, by Markov's Inequality, we know if we only consider the randomness of the labels, with probability at least $1-\frac{1}{10}$, the expected $p$th power loss of all points is contained in $((1-\mu)^p-\frac{\epsilon}{2},(1-\mu)^p+\frac{\epsilon}{2})$, which completes the proof.

Note that $(1-\frac{1-\epsilon'}{2p})^p-(1-\frac{1}{2p})^p>3\epsilon$, therefore the algorithm correctly guesses the mean of the coin with probability at least $(1-\frac{1}{10})\cdot\frac{2}{3}=\frac{3}{5}$.
\end{proof}

%Our goal of the testing algorithm is to approximate the accuracy of the $k$-nearest majority vote algorithm given $X,d,f$ and $N$. We assume that the distribution $\calD$ is unknown to the algorithm and the function $f$ is given by active access. 

%A natural question to ask is whether there is a better algorithm achieving less query complexity. By the following reduction theorem, we'll show a negative answer to this question under the assumption that the ``naive'' algorithm for $\aga$ cannot be improved in query complexity. The theorem also leads to an unconditional sample complexity lower bound for approximating the accuracy of $k$-nearest majority vote indicating the linear growth with respect to $k$ in the sample complexity is unimprovable.

%$\aga(\mathbf A, \gamma,\epsilon)$ has a simple algorithm requiring $O(\frac{1}{\gamma^2\epsilon^2}\log\frac{1}{\epsilon})$ queries as follows. The algorithm randomly picks $O(\frac{1}{\epsilon^2})$ arms. For each of the picked arms, the algorithm queries each arm $O(\frac{1}{\gamma^2}\log\frac{1}{\epsilon})$ times and think of it as ``good'' if more than half of the results are positive and ``bad'' otherwise. The algorithm outputs the fraction of the ``good'' arms among the picked arms.

%The following is the reduction theorem. Note that the theorem is stronger in the sense that the algorithm $A$ has query access and knows the distribution to be the uniform distribution $\calU$.
\begin{theorem}
\label{thm:reductionfrombandit}There exists a positive constant $c$ such that for any fixed $k\in \mathbb{N}^*$, $\epsilon\in (0,\frac{1}{4})$ and oracle $M$ using any fixed tie-breaking mechanism, if there is a $\thard_{\calU}(f\qu,\epsilon,S,k)$ tester using at most $q$ queries in the worst case, then there is an $\aga(\mathbf A,\gamma,2\epsilon)$ algorithm using at most $O(q)$ queries in the worst case where $\gamma=\min\left\{\frac{1}{2},c\cdot \sqrt{\frac{\log\frac{1}{\epsilon}}{k}}\right\}$.
\end{theorem}
\begin{proof}
Since the success probability can be boosted by repetition, we only show an $\aga(\mathbf A,\gamma,2\epsilon)$ algorithm with success probability at least $\frac{3}{5}$. Given any instance of $\aga(\mathbf A,\gamma,2\epsilon)$ with total number of arms equal to $n$, the algorithm constructs a ground set $X$ and the distance metric $\D$ on it to form a $\mv$ instance in the following way. We first choose $b=\lceil\frac{3}{\epsilon}\rceil,N=\lceil c'\cdot (1+b)n(k+\log\frac{1}{\epsilon})\rceil$ and $m=\lceil\frac{c''N^2}{(1+b)n}\rceil\geq\frac{3N}{(1+b)n\epsilon}$. Here, $c'$ and $c''$ are sufficiently large constants. $X$ consists of $n$ identical stars, each corresponds to an arm, with the distances between stars to be very large. Each star consists of $m$ centers and $bm$ leaves. Each center $C$ has a distance $\D_C\in(1,2)$ to every leaf in the same star and different centers have different values of $\D_C$ to avoid ties. The distance between each pair of leaves in the same star is 2 and the distance between each pair of centers in the same star is 1. 

The algorithm then simulates the tester $\thard_{\calU}(f\qu,\epsilon,S,k)$ on this $\mv$ instance without knowing $f$ beforehand. Every time the tester queries the label of a new example, it simulates the result as follows. If the example being queried is a leaf, the result is 0. If the example being queried is a center, the result is obtained to be the same result of an independent query to the corresponding arm. Finally, the algorithm outputs $\hat\alpha n$ when the $\thard_{\calU}(f\qu,\epsilon,S,k)$ tester outputs $\alpha$.

%We run algorithm $A$ on ground set $X$ with distance metric $\D$ without knowing the labels $f$ beforehand. Each time $A$ queries the label of a new example, we answer in the following manner. If the example being queried is a leaf, we always answer 0, and if the example being queried is a center of the $i$th star, we query the $i$th arm and answer the result of the coin toss. Finally we output what $A$ outputs. We claim that this is an $\aga(\mathbf A,\gamma,2\epsilon)$ algorithm.
Now we show that the above is an $\aga(\mathbf A,\gamma,2\epsilon)$ algorithm with success probability at least $\frac{3}{5}$. The process of the algorithm, by interchanging the randomness of the labels (arms) and the internal randomness of the $\thard_{\calU}(f\qu,\epsilon,S,k)$ tester, can be viewed in the way that the true labels $f$ are determined before we run the $\thard_{\calU}(f\qu,\epsilon,S,k)$ tester. The leaves all have labels 0 and each center is independently labeled 0 or 1 according to the result of a query to the corresponding arm. After the labels $f$ are decided, the $\thard_{\calU}(f\qu,\epsilon,S,k)$ tester is then simulated  to approximate the error rate of the $\mv$ instance up to additive error $\epsilon$ with success probability at least $\frac{2}{3}$. 

Let's say a star is good (bad) if it corresponds to a good (bad) arm. Suppose there are $\xi n$ good arms, and thus $\xi n$ good stars. Note that there are $(1+b)mn=\Omega(N^2)$ points in the ground set, we can make sure with probability at least $1-\frac{1}{20}$ that no two unlabeled examples lie on the same point, on which the following discussion is conditioned. Let's first fix a star $R$ whose corresponding arm has mean $\mu$. Because each random example has probability $\frac{1}{(1+b)n}$ to lie in the centers of $R$ and $N\geq c'\cdot (1+b)n(k+\log\frac{1}{\epsilon})$, therefore by choosing a sufficiently large $c'$, we can make sure with probability at most $\frac{\frac{\epsilon}{120}}{1-\frac{1}{20}}$ that in the unlabeled sample pool, there are less than $k$ examples lying at the centers of $R$. Therefore, by a sufficiently large choice of $c$, among those unlabeled examples lying at the centers of $R$, we can make sure that with probability at least $(1-\frac{\frac{\epsilon}{120}}{1-\frac{1}{20}})(1-\frac{\epsilon}{200})\geq 1-\frac{\frac{\epsilon}{60}}{1-\frac{1}{20}}$, the average of the labels of the $k$ examples with smallest $\D_C$ is contained in $(\mu-\gamma,\mu+\gamma)$, or $R$ is \emph{satisfied}. By Markov's Inequality, with probability at least $1-\frac{\frac{1}{20}}{1-\frac{1}{20}}$, or $1-\frac{1}{10}$ if we unwrap the conditional probability of $1-\frac{1}{20}$, at least a $(1-\frac{\epsilon}{3})$ fraction of all the $n$ stars are satisfied. In a satisfied star, any leaf that is not in the unlabeled pool has $L^1$ error 1 if the star is good and $L^1$ error 0 if the star is bad. Note that there are at most $N$ leaves in the unlabeled pool, contributing at most an $\frac{N}{(1+b)mn}\leq\frac{\epsilon}{3}$ fraction of the total number of points. Also there are only $mn$ centers in total, contributing at most an $\frac{mn}{(1+b)mn}\leq\frac{\epsilon}{3}$ fraction of the total number of points. Therefore, with probability at least $1-\frac{1}{10}$, the average error of all points is contained in $[\xi-\epsilon,\xi+\epsilon]$, which implies that with probability at least $(1-\frac{1}{10})\cdot\frac{2}{3}=\frac{3}{5}$, $\hat\alpha\in[\xi-2\epsilon,\xi+2\epsilon]$. \end{proof}%Let's first consider the average error rate for points in satisfied stars.  , and in this case, all the leaves of each of these stars that are not in the unlabeled pool $S$ has error rate 1 if the star is good and 0 if the star is bad. The total number of the leaves in the unlabeled pool $S$ and the centers is upper bounded by the size $N$ of the pool plus $m$, which contributes only a $\frac{N+m}{(b+1)m}<\frac{\epsilon}{3}$ fraction of the total number of points. Therefore, with probability at least $1-\frac{1}{10}$, the average $p$th power loss of all points is contained in $((1-\mu)^p-\frac{\epsilon}{2},(1-\mu)^p+\frac{\epsilon}{2})$. 

%To show the claim, given the fraction of good arms to be $\xi$, we only need to show that if the labels of the centers are chosen independently at random such that each center of a star whose corresponding arm has mean $p$ has probability $p$ being labeled 1 and probability $1-p$ being labeled 0, then with probability $\frac{9}{10}=\frac{\frac{3}{5}}{\frac{2}{3}}$, the $\mv$ instance has error rate in $[\xi-\epsilon,\xi+\epsilon]$. In fact, with probability at least $1-\frac{\epsilon^2}{120}$, each star satisfies that there are more than $k$ training data located in its center (by the choice of $N$ and $c'$), more than half of the $k$-nearest are labeled 1 for good stars and less than half of the $k$-nearest are labeled 1 for bad stars (by the choice of $\gamma$ and $c$). By Markov's Inequality, with probability $1-\frac{\epsilon}{20}$, more than $1-\frac{\epsilon}{6}$ fraction of the stars satisfy the previous property. Conditioned on this happening, the instance has error rate in $[\xi-\frac{\epsilon}{2},\xi+\frac{\epsilon}{2}]$ because $b\geq\frac{6}{\epsilon}$ and $N\leq\frac{\epsilon}{6}\cdot mnb$. Then by Markov's inequality, with probability at least $1-\frac{1}{10}$ of the labeling, with probability at least $1-\frac{\epsilon}{2}$ of the sampling has error rate in $[\xi-\frac{\epsilon}{2},\xi+\frac{\epsilon}{2}]$, which means that with probability at least $\frac{9}{10}$ of the labeling, the error rate of the instance is in $[\xi-\epsilon,\xi+\epsilon]$. 

The above theorem shows that a query complexity lower bound for $\aga(\mathbf A, \gamma,\epsilon)$ can imply a query complexity lower bound for $\thard_{\calU}(f\qu,\epsilon,S,k)$. $\aga(\mathbf A, \gamma,\epsilon)$ has a simple algorithm requiring $O(\frac{1}{\gamma^2\epsilon^2}\log\frac{1}{\epsilon})$ queries as follows. The algorithm randomly picks $O(\frac{1}{\epsilon^2})$ arms. For each of the picked arms, the algorithm queries it $O(\frac{1}{\gamma^2}\log\frac{1}{\epsilon})$ times and thinks of it as ``good'' if more than half of the results are positive and ``bad'' otherwise. The algorithm outputs the fraction of ``good'' arms among the picked arms.

If we assume the simple $O(\frac{1}{\gamma^2\epsilon^2}\log\frac{1}{\epsilon})$ query complexity for $\aga$ is not improvable, then Theorem \ref{thm:reductionfrombandit} implies that the $O(\frac{k}{\epsilon^2})$ query complexity for $\thard$ is also not improvable. In other words, if for every sequences $\epsilon_n\rightarrow 0$ and $\gamma_n\rightarrow 0$, there exists a positive constant $c$ such that $\aga(\mathbf A,\epsilon_i,\gamma_i)$ needs at least $c\cdot \frac{1}{\gamma_i^2\epsilon_i^2}\log\frac{1}{\epsilon_i}$ queries in the worst case, then according to Theorem \ref{thm:reductionfrombandit}, we know for any sequences $\{k_n\},\{\epsilon_n\}$ such that $\epsilon_n\rightarrow 0,\frac{k_n}{\log\frac{1}{\epsilon_n}}\rightarrow\infty$, there exists a positive constant $c'$ such that the tester $\thard_{\calU}(f\qu,\epsilon,S,k)$ for $\mv$ algorithms needs at least $c'\cdot\frac{k_i}{\epsilon_i^2}$ queries in the worst case.

%Though we are not able to show an $\Omega(\frac{1}{\gamma^2\epsilon^2}\log\frac{1}{\epsilon})$ lower bound for approximating the fraction of good arms, by considering it's special case, \emph{counting good arms}, we can easily get an $\Omega(\frac{1}{\gamma^2\epsilon})$ lower bound, which implies by Theorem \ref{thm:reductionfrombandit} an $\Omega(\frac{k}{\epsilon\log\frac{1}{\epsilon}})$ lower bound for approximating the accuracy of $\mv$. Specifically, we show the following theorem.

The following theorem states an unconditional lower bound $\Omega(\frac{k}{\epsilon\log\frac{1}{\epsilon}})$ for the query complexity of $\thard$, implying that the linear growth with respect to $k$ in the query complexity of $\thard$ can't be improved.
\begin{theorem}
\label{thm:mv}
There exists a positive constant $c$ such that for any fixed $k\in\mathbb{N}^*,\epsilon\in(0,\frac{1}{4})$ and oracle $M$ using any fixed tie-breaking mechanism, $\thard_{\calU}(f\qu,\epsilon,S,k)$ requires at least $c\cdot\frac{k}{\epsilon\log\frac{1}{\epsilon}}$ queries in the worst case.
\end{theorem}
Before proving Theorem \ref{thm:mv}, we first show a query complexity lower bound for $\cga$.
\begin{lemma}
\label{lm:counting}
There exists a universal constant $c$ such that for any fixed $\gamma\in(0,\frac{1}{2}]$ and $n\in\mathbb N^*$, $\cga(\mathbf A,\gamma)$ requires at least $c\cdot\frac{n}{\gamma^2}$ queries in the worst case, where $n$ is the number of arms in $\mathbf A$.
\end{lemma}

\begin{proof}[Proof of Lemma \ref{lm:counting}]
Obviously, $n$ is a query complexity lower bound since we need to query each arm at least once. So in the rest of the proof, we assume $\gamma<\frac{1}{4}$. We use $G$ to denote the good arm with mean $\frac{1}{2}+\gamma$ and $B$ to denote the bad arm with mean $\frac{1}{2}-\gamma$. Then, $\mathrm{KL}(G,B)=O(\gamma^2)$. We claim a stronger fact that for any $0\leq q\leq n$ and any instance consisting of $q$ $G$'s and $n-q$ $B$'s, $\cga(\mathbf A,\gamma)$ needs at least $c\cdot\frac{1}{\gamma^2}$ queries on \emph{every} of the $n$ arms. By symmetry between ``good'' and ``bad'', we only show that every $G$ arm needs to be queried at least $c\cdot\frac{1}{\gamma^2}$ times. The reason is as follows. Suppose $\mathbf A=(A_1,A_2,\cdots,A_n)$ in which $A_i=G$ for $1\leq i\leq q$ and $A_i=B$ otherwise. We define $\mathbf A'=(A_1',A_2',\cdots,A_n')$ in which $A'_i=G$ for $1\leq i\leq p-1$ and $A'_i=B$ otherwise. The only difference between $\mathbf A$ and $\mathbf A'$ is that $A_p=G$ while $A'_p=B$. We use $\calE$ to denote the event that $\cga(\mathbf A,\gamma)$ outputs $p$. By Lemma \ref{lm:changeofdistribution}, we know $\mathbb E[\tau_p]\cdot O(\gamma^2)\geq D(\frac{2}{3},\frac{1}{3})=\Omega(1)$ and thus $\mathbb E[\tau_p]=\Omega(\frac{1}{\gamma^2})$. For similar reasons, we can show for all $1\leq i\leq p$ that $\mathbb E[\tau_i]=\Omega(\frac{1}{\gamma^2})$, which completes the proof.
\end{proof}

\begin{proof}[Proof of Theorem \ref{thm:mv}]
Lemma \ref{lm:counting} immediately implies the existence of a positive constant $c'$ such that for any fixed $\epsilon\in(0,\frac{1}{2})$ and $\gamma\in(0,\frac{1}{2}]$, $\aga(\mathbf A, \gamma,\epsilon)$ requires at least $c'\cdot\frac{1}{\gamma^2\epsilon}$ queries in the worst case by choosing $n=\lceil\frac{1}{2\epsilon}\rceil-1$. Then, by Theorem \ref{thm:reductionfrombandit}, we get an $\Omega(\frac{1}{
\left(\min\left\{\frac{1}{2},\sqrt{\frac{
	\log\frac{1}{\epsilon}
}{k}}\right\}\right)^2
}\cdot\frac{1}{2\epsilon})=\Omega(\frac{k}{\epsilon\log\frac{1}{\epsilon}})$ lower bound for $\thard_{\calU}(f\qu,\epsilon,S,k)$ for $k\in\mathbb N^*$ and $\epsilon\in(0,\frac{1}{4})$.
\end{proof}
%\begin{Theorem}

%\end{Theorem}

%% file: acknowledgement.tex
\section*{Acknowledgements}
\label{sec:ac}
This work was supported in part by the National Science Foundation under grant CCF-1525971.  This work was conducted in part while Avrim Blum was at Carnegie Mellon University and while Lunjia Hu was visiting at Carnegie Mellon University and TTI-Chicago.

%% file: union.tex
\section{Distance Approximation for Disjoint Unions of Properties}
\label{sec:union}
In this section, we extend the theorem of \citet{BBBY12} that disjoint unions of testable properties are testable from property testing to distance approximation. 
%\subsection{Active Distance Approximation}
%Distance approximation is a natural generalization of tolerant testing. Given a distribution $\calD$ and a concept class $\calC$ on ground set $X$, the distance of a function $f\in\{0,1\}^X$ to $\calC$ is defined to be $\dist_{\calD}(f,\calC)=\inf\limits_{g\in\calC}\dist_{\calD}(f,g)$. A distance approximation algorithm takes a function $f\in\{0,1\}^X$ and a parameter $\epsilon$ as input, and outputs an estimation $\widehat\dist_{\calD}(f,\calC)$ which is in $[\dist_{\calD}(f,\calC)-\epsilon,\dist_{\calD}(f,\calC)+\epsilon]$ with probability at least $\frac{2}{3}$. The parameter $\epsilon$ is called the \emph{additive error}.

%For an active distance approximation algorithm, similar to an active tolerant testing algorithm, the input function $f$ is given through active access. The success probability can be boosted to $1-\delta$ for any $\delta$ by repeating the algorithm $O(\log\frac{1}{\delta})$ times and choosing the median.

We first introduce the definition of disjoint unions of properties in \citep{BBBY12}. Suppose the ground set $X$ is partitioned as a disjoint union $\bigcup\limits_{i=1}^m X_i$. For every $X_i$, there is a property (concept class) $\calC_i\neq\emptyset$. The disjoint union of these properties is defined to be $\calC=\{f\in\{0,1\}^X:\forall 1\leq i\leq m,f|_{X_i}\in\calC_i\}$.

Let $\calD$ be a distribution over $X$. Suppose the conditional distribution of $\calD$ on $X_i$ is denoted by $\calD_i$ and the probability $\Pr_{x\sim\calD}[x\in X_i]$ is denoted by $p_i$.

\begin{theorem}
Suppose $\epsilon\in(0,\frac{1}{2})$. Suppose there is a $\da_{\calD_i}(f\ac,\frac{\epsilon}{2})$ algorithm for every $1\leq i\leq m$ using at most $q$ queries on $N$ unlabeled examples. Then, there is a $\da_{\calD}(f\ac,\epsilon)$ algorithm using at most $O(\frac{q}{\epsilon^2}\log\frac{1}{\epsilon})$ queries on $O(\frac{mN}{\epsilon}\log\frac{1}{\epsilon})$ unlabeled examples. If the $\da_{\calD_i}(f\ac,\frac{\epsilon}{2})$ algorithm can perform on unknown distributions, then the $\da_{\calD}(f\ac,\epsilon)$ algorithm can also perform on unknown distributions, though we need extra $O(\frac{1}{\epsilon^2})$ unlabeled examples.

%Suppose there is an $\frac{\epsilon}{2}$-distance approximation oracle using at most $q$ queries on at most $U$ unlabeled samples. Then, the distance $\dist_{\mathcal{D}}(f,\mathcal{C})$ can be approximated with success probability at least $\frac{2}{3}$ within additive error $\epsilon+\delta$ using $O(\frac{q}{\delta^2}\log\frac{1}{\delta})$ queries on $O(\frac{mU}{\delta}\log\frac{1}{\delta})$ unlabeled samples. Furthermore, if $\forall 1\leq i\leq m,p_i=\frac{1}{m}$, then the unlabeled sample complexity can be reduced to $O(mU\log\frac{1}{\delta})$.
\end{theorem}

\begin{proof}
The $\da_{\calD}(f\ac,\epsilon)$ algorithm is constructed as follows. The algorithm chooses $s=O(\frac{1}{\epsilon^2})$, receives an unlabeled pool of size $O(\frac{mN}{\epsilon}\log s)$ and independently chooses $s$ indices $i_1,i_2,\cdots,i_s$ from $\{1,2,\cdots,m\}$ according to distribution $\{p_i\}_{1\leq i\leq m}$. This can be achieved by looking at on which $X_i$'s the extra $s$ unlabeled examples are, when the distribution $\calD$ is unknown. Then for each $1\leq j\leq s$, if there are enough $(O(N\log s))$ unlabeled examples lying in $X_{i_j}$, the algorithm repeats $\da_{\calD_{i_j}}(f\ac,\frac{\epsilon}{2})$ for $O(\log s)$ times to calculate an estimator $\widehat\dist_{i_j}$ of the distance from $f$ to $\calC$ on $\calD_{i_j}$ up to an additive error $\frac{\epsilon}{2}$ with success probability at least $1-\frac{1}{9s}$; otherwise, define $\widehat\dist_{i_j}=0$. The final output of the algorithm is $\frac{1}{s}\cdot\sum\limits_{j=1}^s\widehat\dist_{i_j}$.

To prove the correctness of the above algorithm, we first define $\dist_i:=\inf\limits_{g\in\calC}\dist_{\calD_i}(f,g)$ and $\dist:=\inf\limits_{g\in\calC}\dist_{\calD}(f,g)$. Note that $\dist=\sum\limits_{i=1}^mp_i\dist_{i}$. 

For every $1\leq i\leq m$, we further define $\dist_i'=\left\{\begin{array}{ll}\dist_i,&\text{if $p_i\geq\frac{\epsilon}{4m}$}\\0,&\text{if $p_i<\frac{\epsilon}{4m}$}\end{array}\right.$ and $\dist_i''=\left\{\begin{array}{ll}\dist_i,&\text{if $p_i\geq\frac{\epsilon}{4m}$}\\1,&\text{if $p_i<\frac{\epsilon}{4m}$}\end{array}\right.$. Then $\dist-\frac{\epsilon}{4}\leq \sum\limits_{i=1}^mp_i\dist'\leq\sum\limits_{i=1}^mp_i\dist''\leq\dist+\frac{\epsilon}{4}$. By the Chernoff Bound, $s=O(\frac{1}{\epsilon^2})$ is enough to make sure with probability at least $1-\frac{1}{9}$ that $\dist-\frac{\epsilon}{2}<\frac{1}{s}\sum\limits_{j=1}^s\dist'_{i_j}\leq\frac{1}{s}\sum\limits_{j=1}^s\dist''_{i_j}<\dist+\frac{\epsilon}{2}$.

Note that the unlabeled pool has size $O(\frac{mN}{\epsilon}\log s)$, which is enough to make sure that with probability at least $1-\frac{1}{9}$, for every $i_j$ with $p_{i_j}\geq\frac{\epsilon}{4m}$, there are enough ($O(N\log s)$) unlabeled examples lying in $X_{i_j}$. Therefore, with probability at least $(1-\frac{1}{9})(1-s\cdot\frac{1}{9s})\geq 1-\frac{2}{9}$, for all $i_j$ such that $p_{i_j}\geq \frac{p}{4m}$, it holds that $|\widehat \dist_{i_j}-\dist_{i_j}|\leq\frac{\epsilon}{2}$.

Finally, by the Union Bound, we know with probability at least $1-\frac{1}{3}$, it holds that $\dist-\epsilon<\frac{1}{s}\sum\limits_{j=1}^s\dist'_{i_j}-\frac{\epsilon}{2}\leq \frac{1}{s}\sum\limits_{j=1}^s\widehat\dist_{i_j}\leq\frac{1}{s}\sum\limits_{j=1}^s\dist''_{i_j}+\frac{\epsilon}{2}<\dist+\epsilon$.
\end{proof}